\documentclass{article} 
\usepackage{nips15submit_e,times}

\usepackage{calc}
\usepackage{amsmath, amssymb, amsthm} 
\usepackage{parskip}
\usepackage{color}
\usepackage{epsfig}
\usepackage{hyperref}
\usepackage{url}
\usepackage{subfigure}
\usepackage{verbatim}
\usepackage{rotating}
\usepackage{multicol}
\usepackage{ifthen}
\usepackage{graphicx} 
\usepackage{subfigure} 
\usepackage{enumitem} 
\usepackage[numbers]{natbib}
\usepackage{algorithm}
\usepackage{algorithmic}
\usepackage{kky}
\usepackage{robinsDefns}

\newcount\Comments
\definecolor{darkgreen}{rgb}{0,0.5,0}                                                                                                                                                                                                                                                                                   
\definecolor{darkred}{rgb}{0.7,0,0}                                                                                                                                                                                                                                                                                     
\definecolor{teal}{rgb}{0.3,0.8,0.8}
\newcommand{\kibitz}[2]{\ifnum\Comments=1\textcolor{#1}{#2}\fi}

\newcommand{\toworkon}[1]{}

\newboolean{istwocolumn}
\setboolean{istwocolumn}{false}
\newboolean{isabridged}
\setboolean{isabridged}{false}

\title{{\Large Influence Functions for Machine Learning: Nonparametric 
  Estimators for Entropies, Divergences and Mutual Informations}}

\author{
David S.~Hippocampus\thanks{ Use footnote for providing further information
about author (webpage, alternative address)---\emph{not} for acknowledging
funding agencies.} \\
Department of Computer Science\\
Cranberry-Lemon University\\
Pittsburgh, PA 15213 \\
\texttt{hippo@cs.cranberry-lemon.edu} \\
\And
Coauthor \\
Affiliation \\
Address \\
\texttt{email} \\
\AND
Coauthor \\
Affiliation \\
Address \\
\texttt{email} \\
\And
Coauthor \\
Affiliation \\
Address \\
\texttt{email} \\
\And
Coauthor \\
Affiliation \\
Address \\
\texttt{email} \\
(if needed)\\
}

\begin{document}

\maketitle


\newcommand{\imarrwthree}{2.1in}
\newcommand{\imhspthree}{-.2in}
\newcommand{\imcaptionspace}{-.05in}
\newcommand{\imbelowspace}{0in}
\newcommand{\imtextspace}{-.1in}

\newcommand{\insertFigToyExamples}{
\begin{figure*}
\centering
\subfigure[]{
  \includegraphics[width=\imarrwthree]{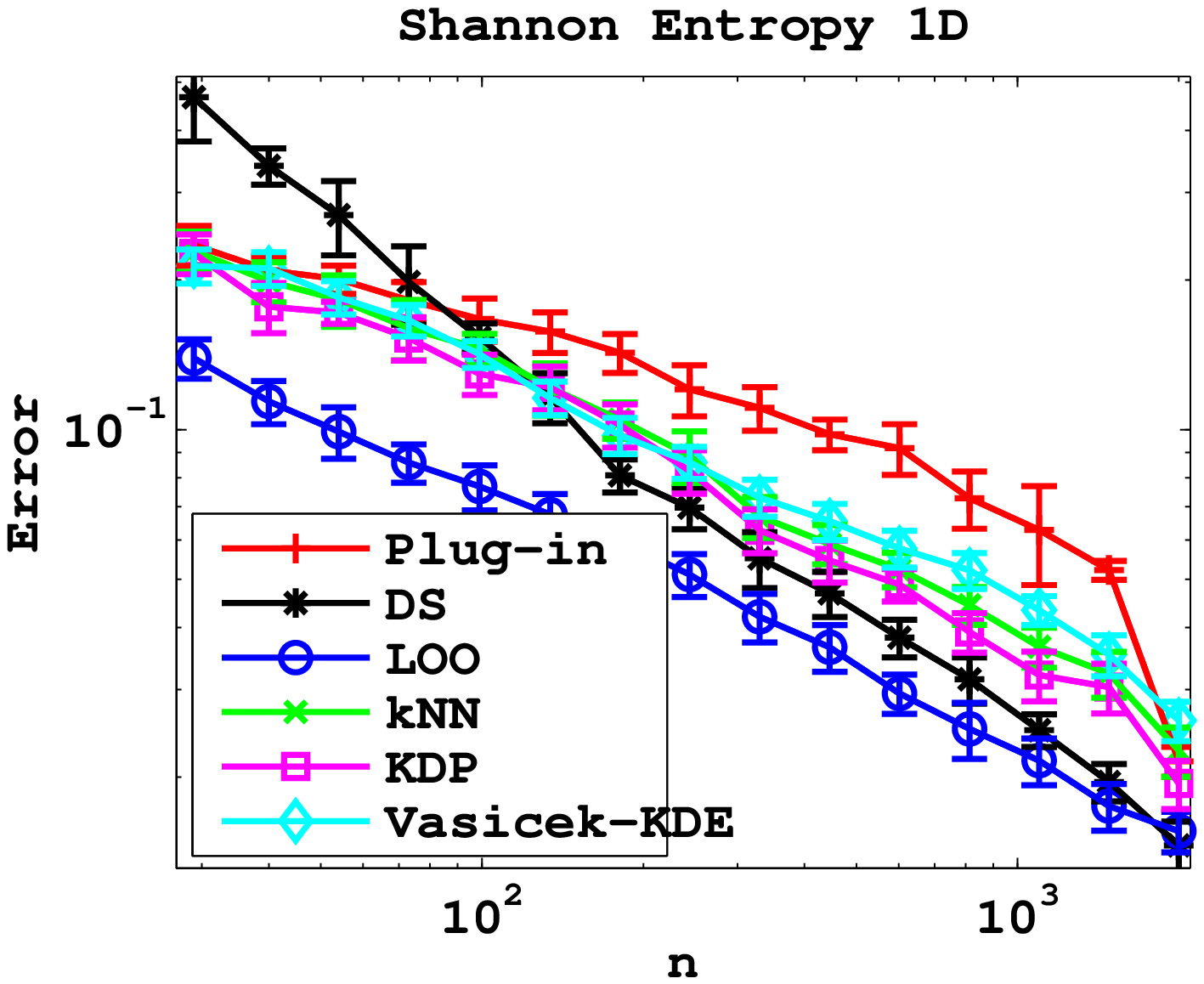} \hspace{\imhspthree}
  \label{fig:entropy1D}
}
\subfigure[]{
  \includegraphics[width=\imarrwthree]{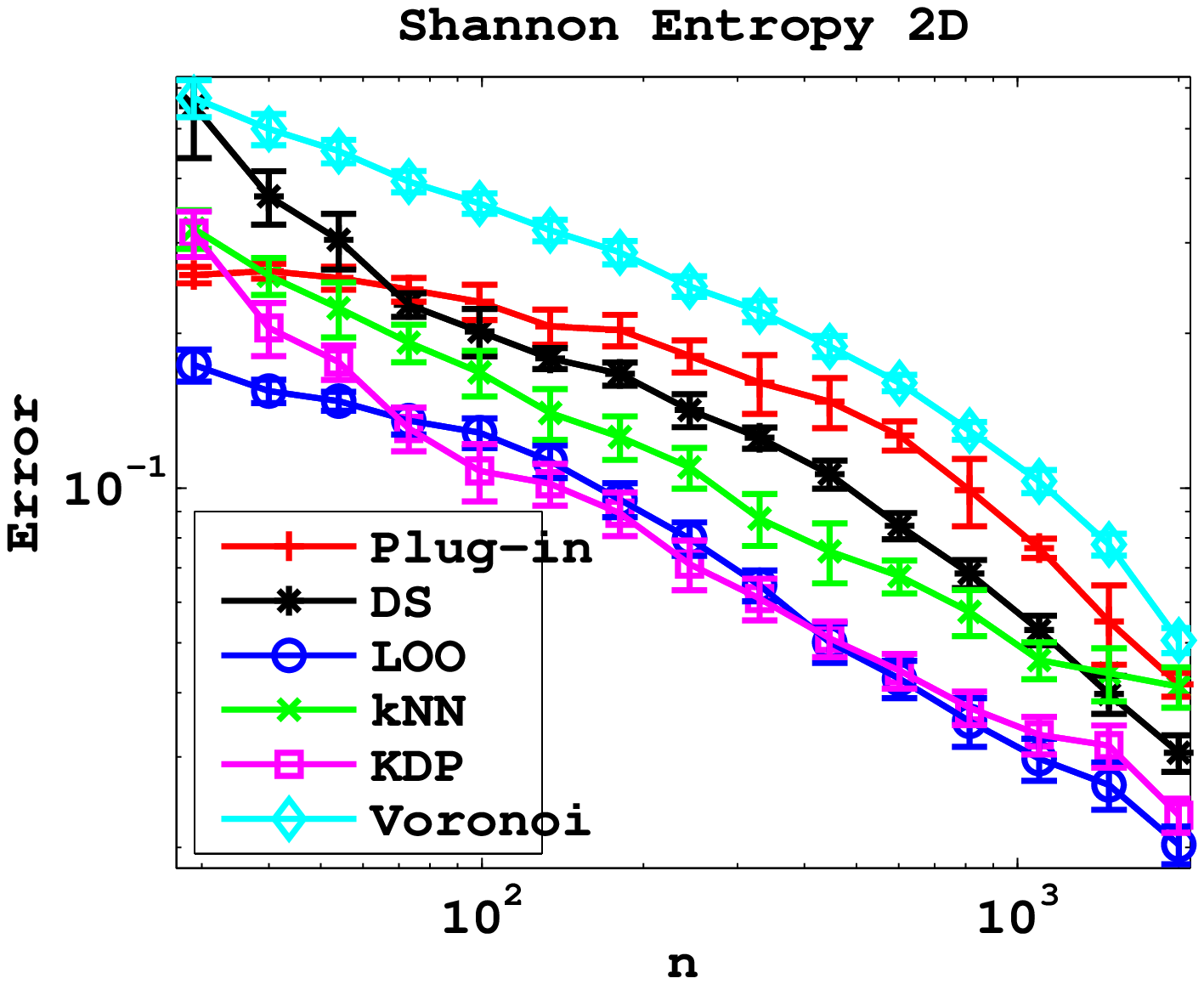} \hspace{\imhspthree}
  \label{fig:entropy2D}
}
\subfigure[]{
  \includegraphics[width=\imarrwthree]{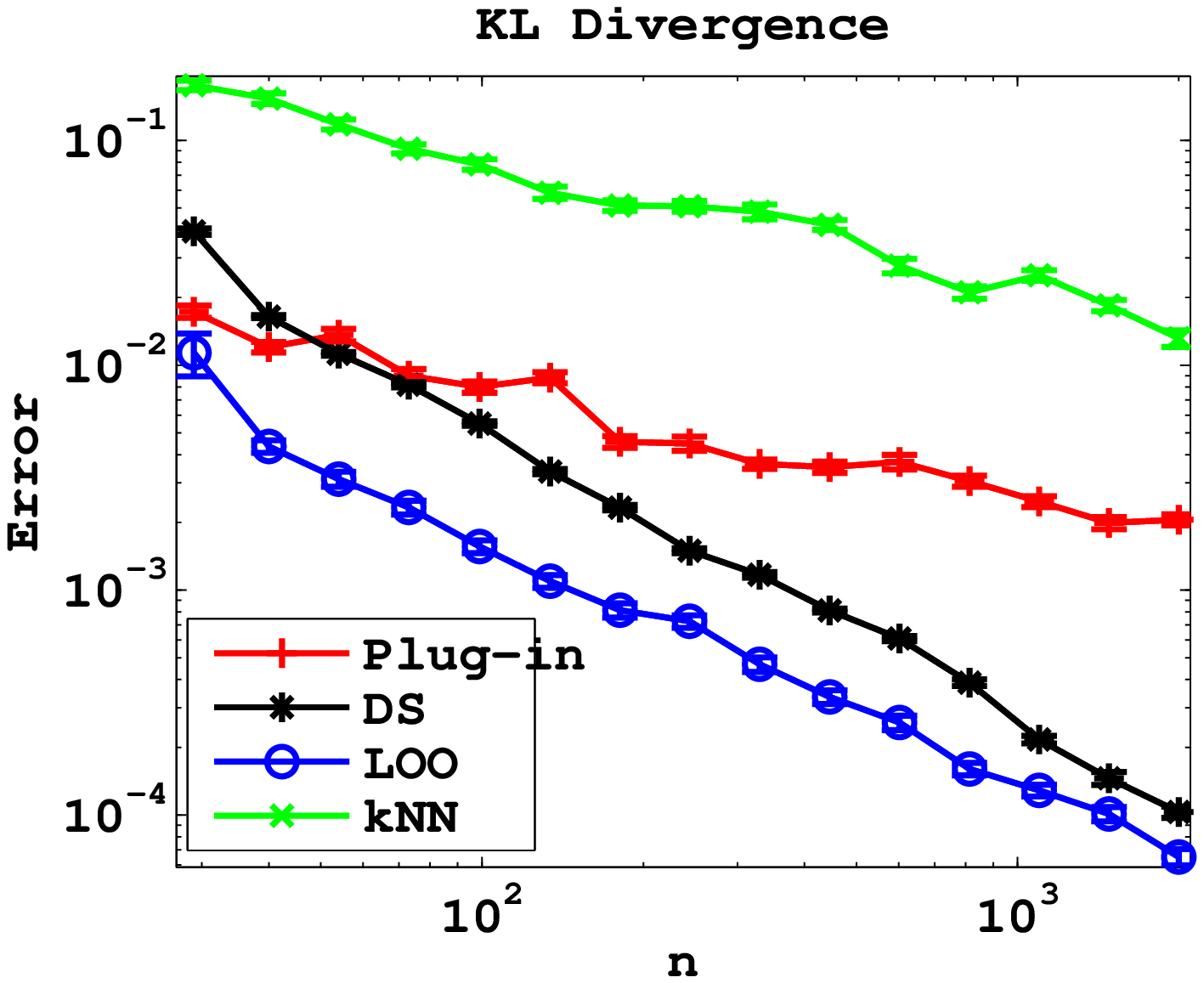}
  \label{fig:kl1D}
} 
\\[\imcaptionspace]
\subfigure[]{
  \includegraphics[width=\imarrwthree]{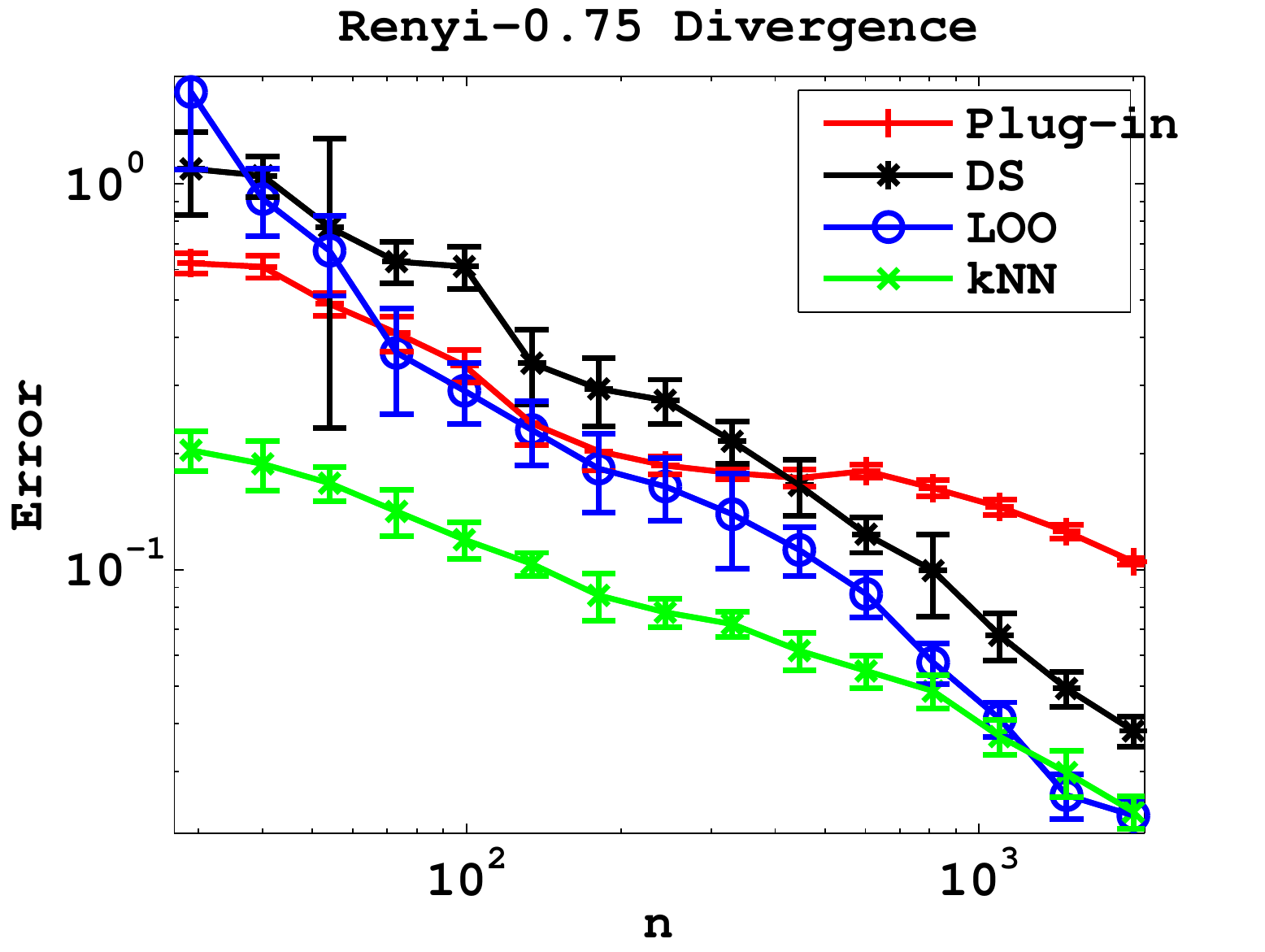} \hspace{\imhspthree}
  \label{fig:renyi1D}
}
\subfigure[]{
  \includegraphics[width=\imarrwthree]{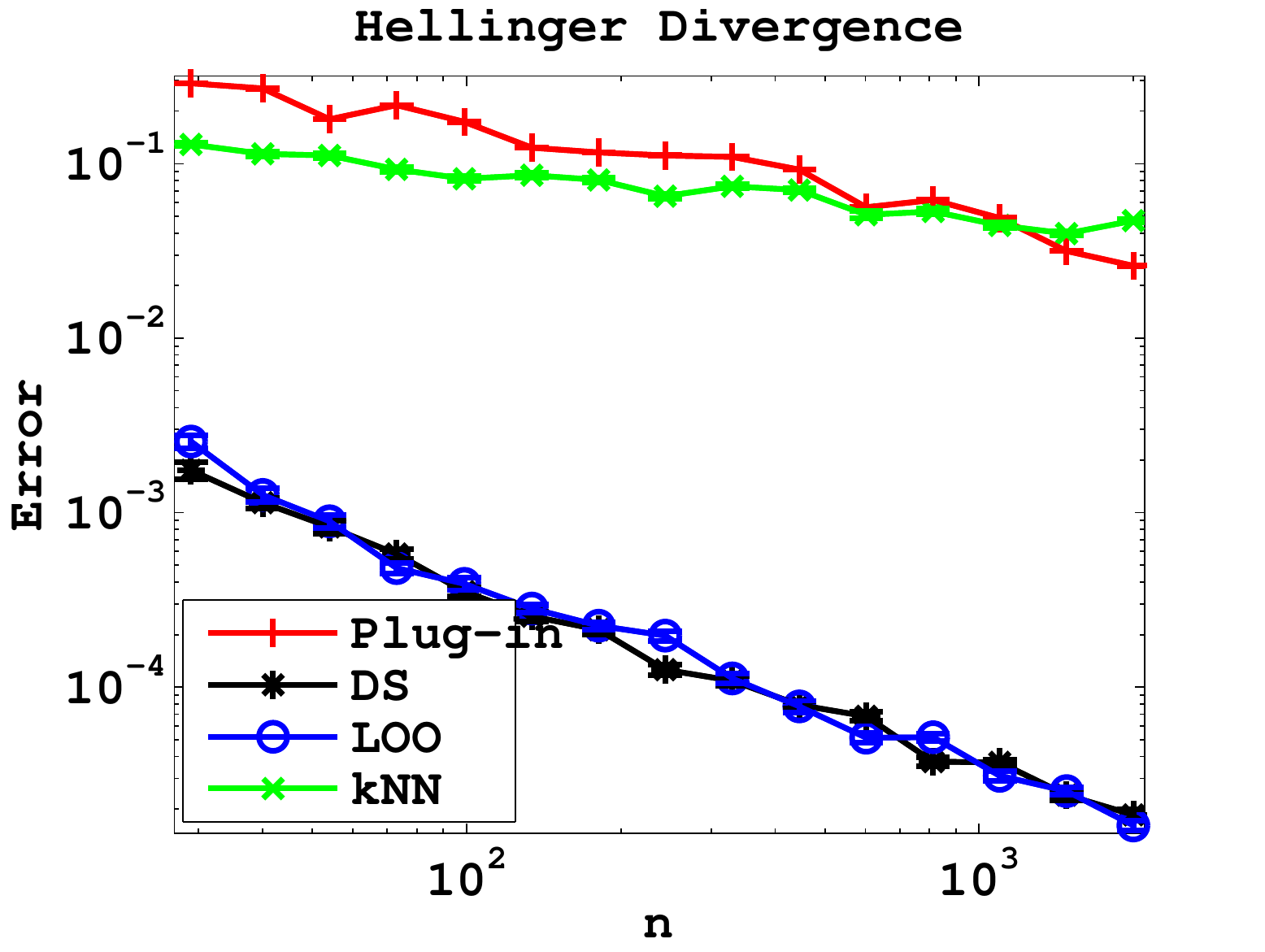} \hspace{\imhspthree}
  \label{fig:hellinger2D}
} 
\subfigure[]{
  \includegraphics[width=\imarrwthree]{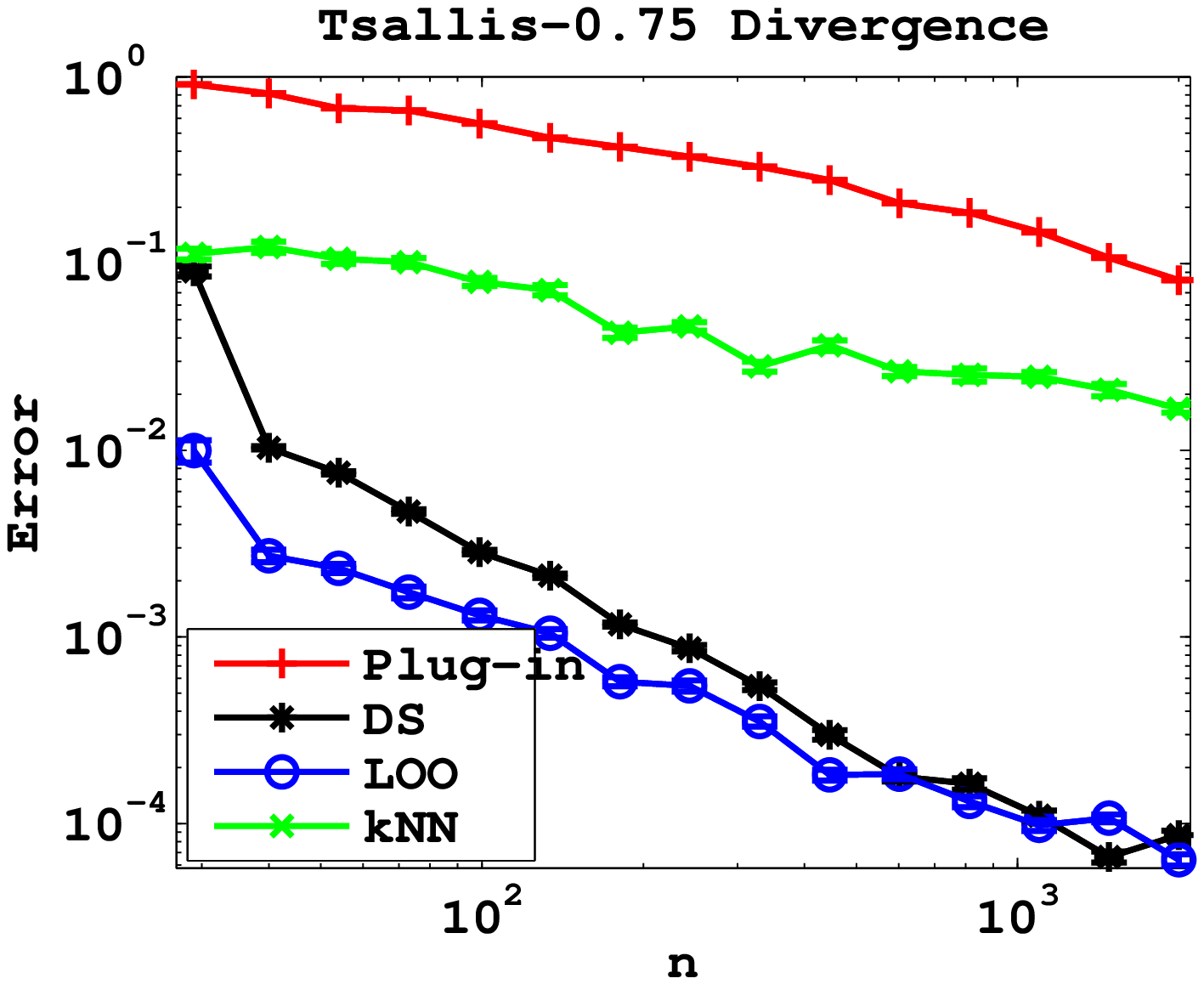}
  \label{fig:tsalllis2D}
}
\\[\imcaptionspace]
\subfigure[]{
  \includegraphics[width=\imarrwthree]{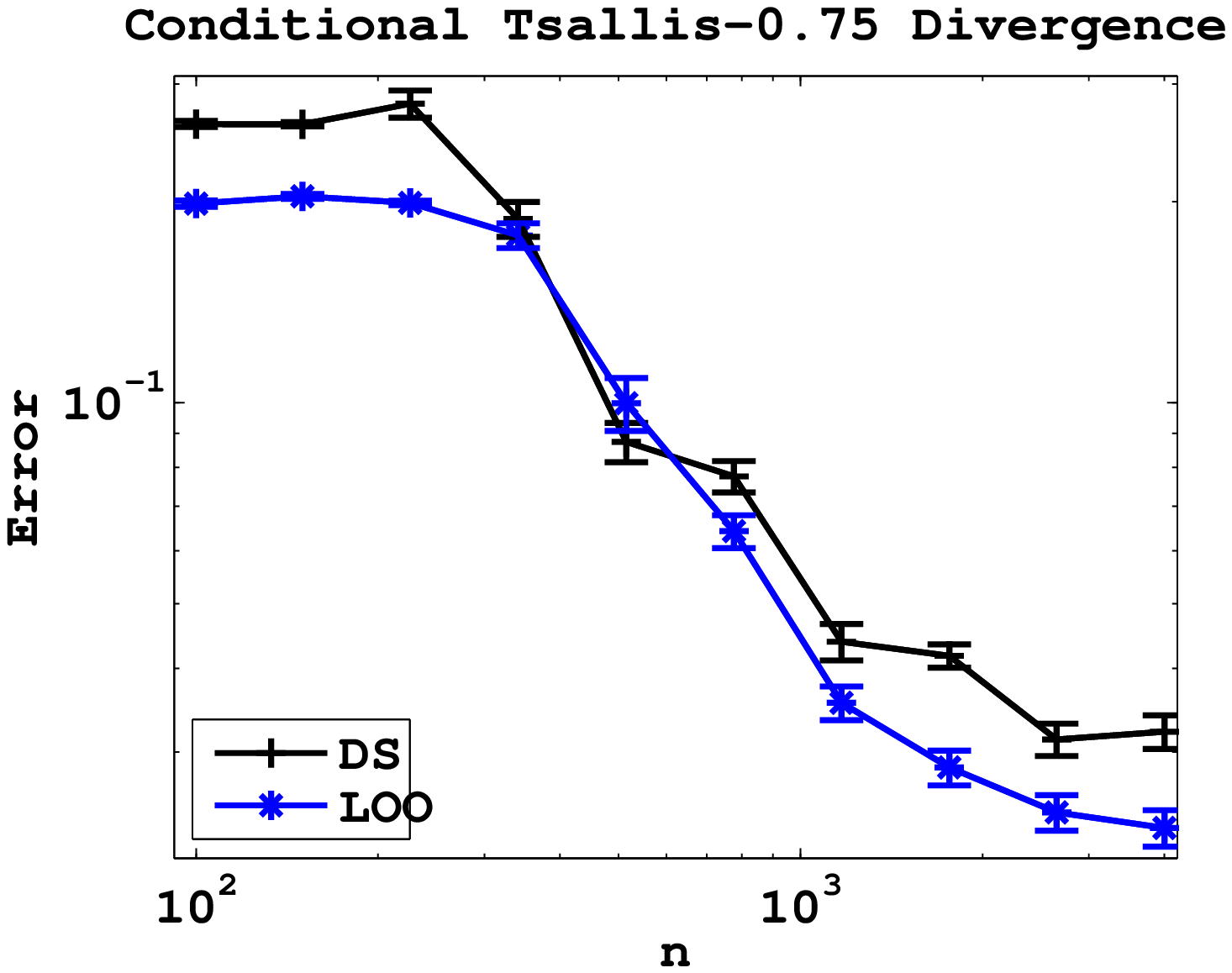} \hspace{\imhspthree}
  \label{fig:condTsallis4D}
}
\subfigure[]{
  \includegraphics[width=\imarrwthree]{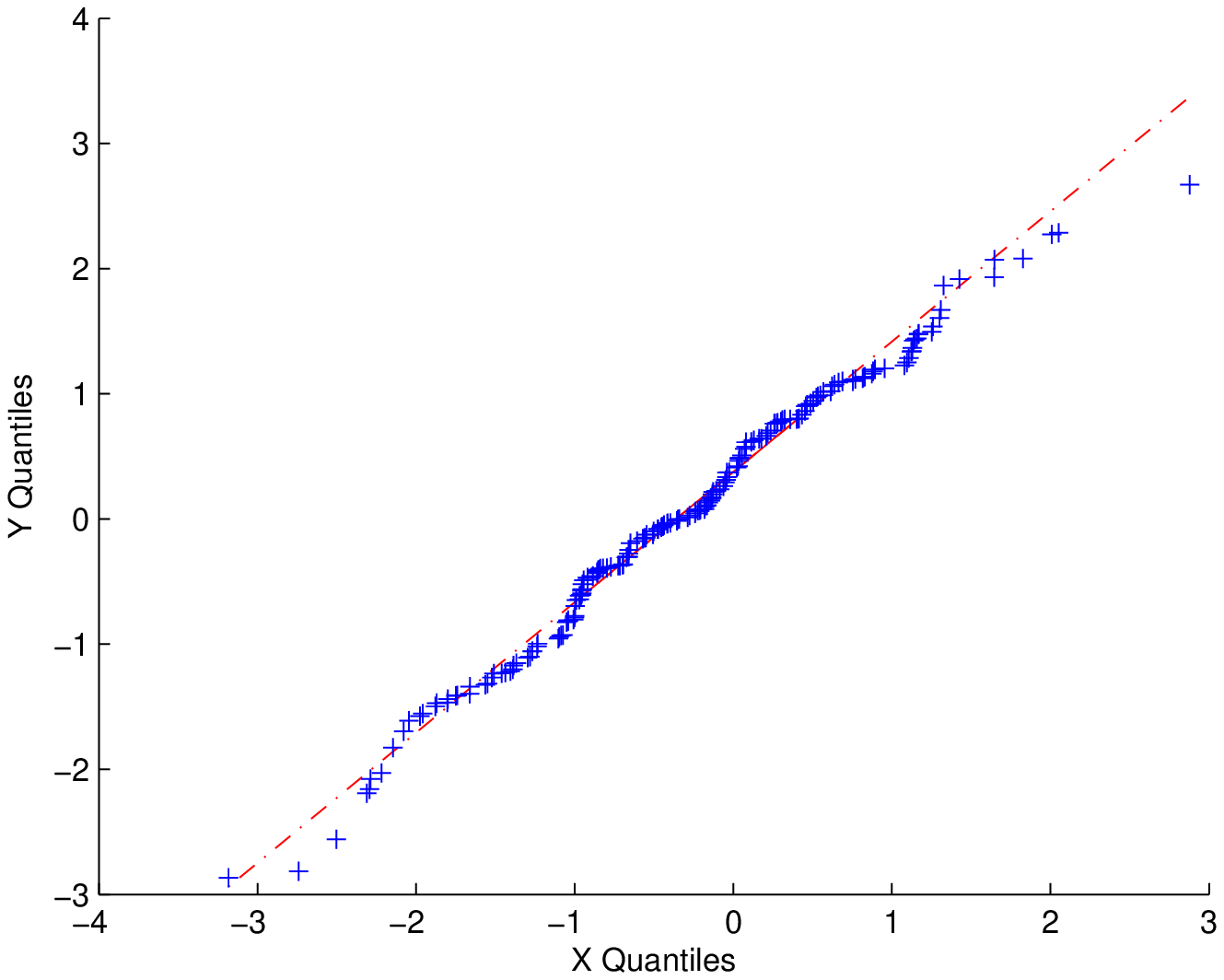} \hspace{\imhspthree}
  \label{fig:dsHellingerAN}
}
\subfigure[]{
  \includegraphics[width=\imarrwthree]{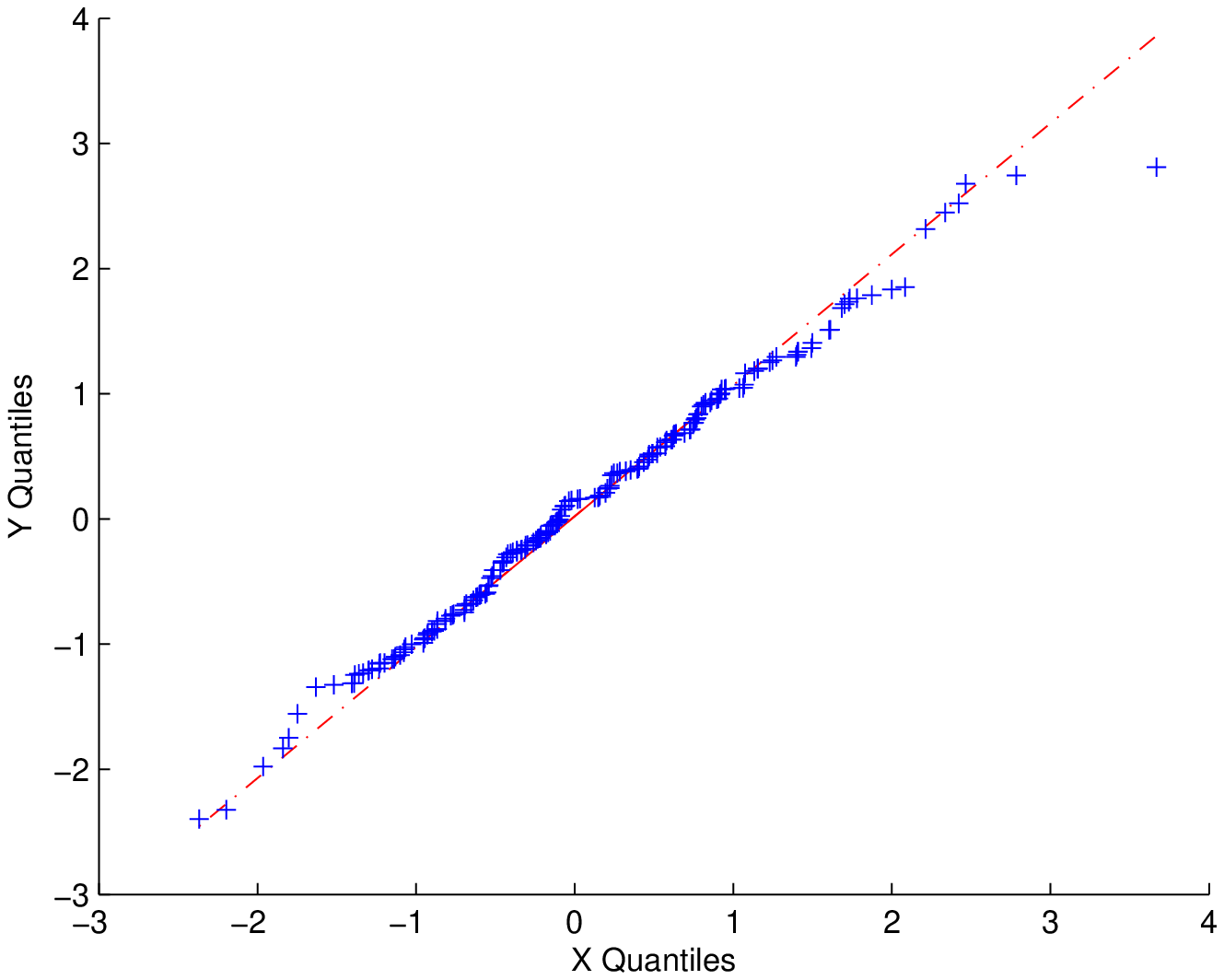}
  \label{fig:looHellingerAN}
} \\
\vspace{\imcaptionspace}
\caption[]{ \small 
Comparison of the \dss and \loos estimators against alternatives on different 
functionals. The $y$-axis is the error $|\Tf - T(f,g)|$ and the $x$-axis is the
number of samples. All figures were produced by averaging over $50$ experimens. 
In Fig~\subref{fig:condTsallis4D}: note that there are no known estimators for the
conditional Tsallis Divergence. Further, since this is a $4$-dimensional problem the
plug-in estimator is computationally intractable due to numerical integration.
Figs~\subref{fig:dsHellingerAN},~\subref{fig:looHellingerAN}: QQ plots obtained using
$4000$ samples for Hellinger divergence estimation in $4$ dimensions using the \dss
and \loos estimators respectively.
}
\label{fig:simulations}
\vspace{\imtextspace}
\end{figure*}
}

\newcommand{\insertFigToyOne}{
\begin{figure*}
\centering
\subfigure{
  \includegraphics[width=\imarrwthree]{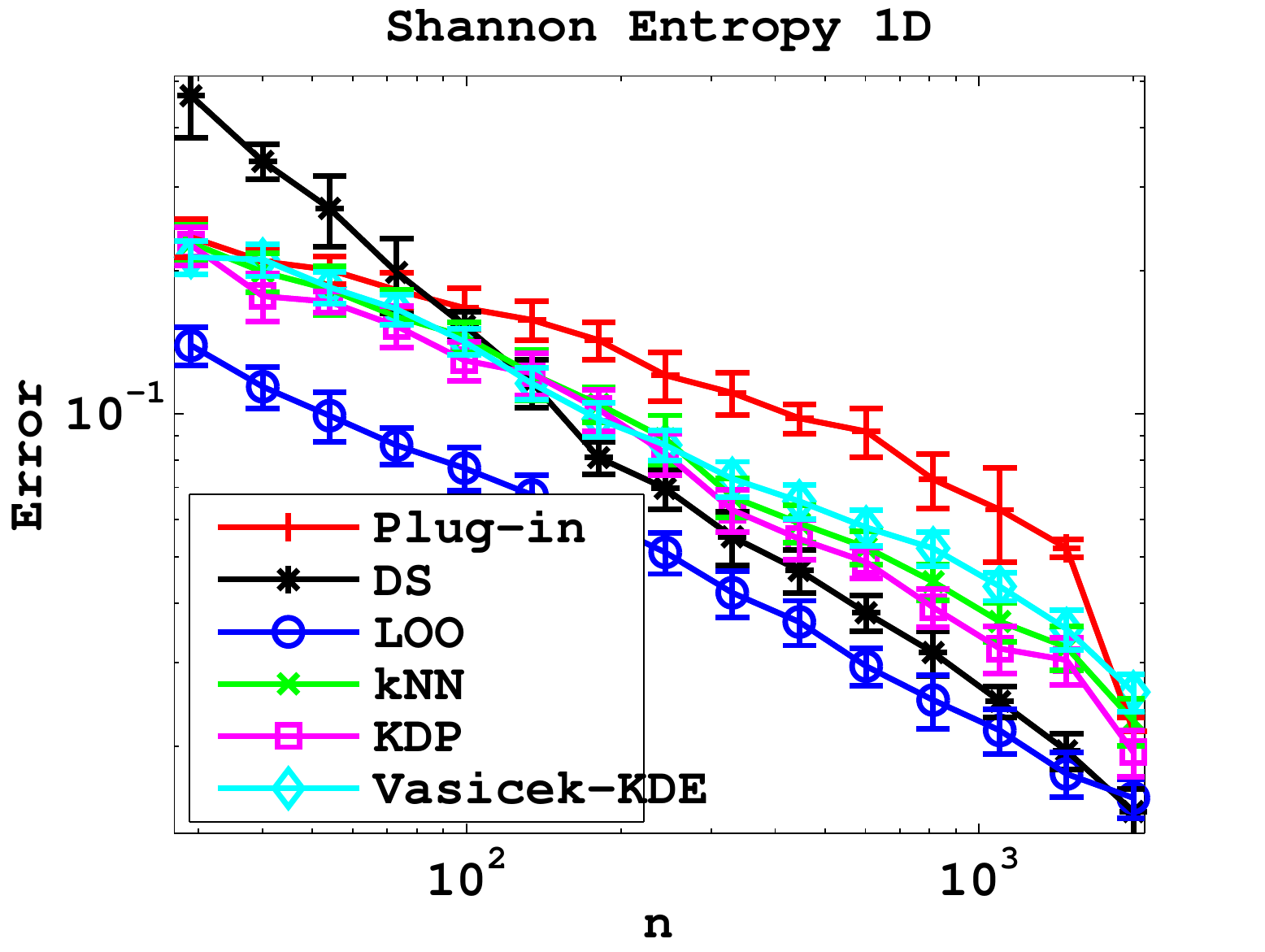} \hspace{\imhspthree}
  \label{fig:entropy1D}
}
\subfigure{
  \includegraphics[width=\imarrwthree]{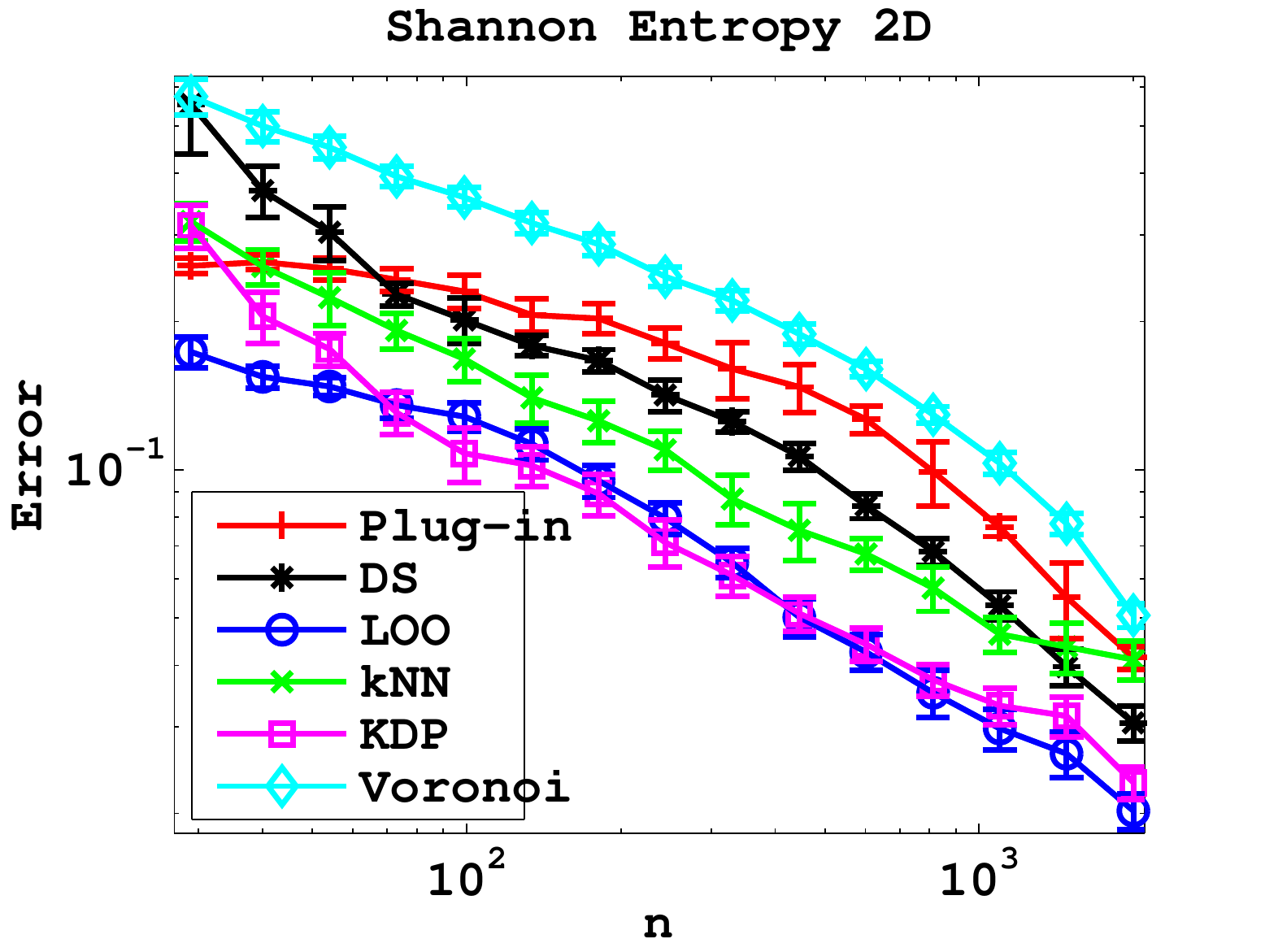} \hspace{\imhspthree}
  \label{fig:entropy2D}
}
\subfigure{
  \includegraphics[width=\imarrwthree]{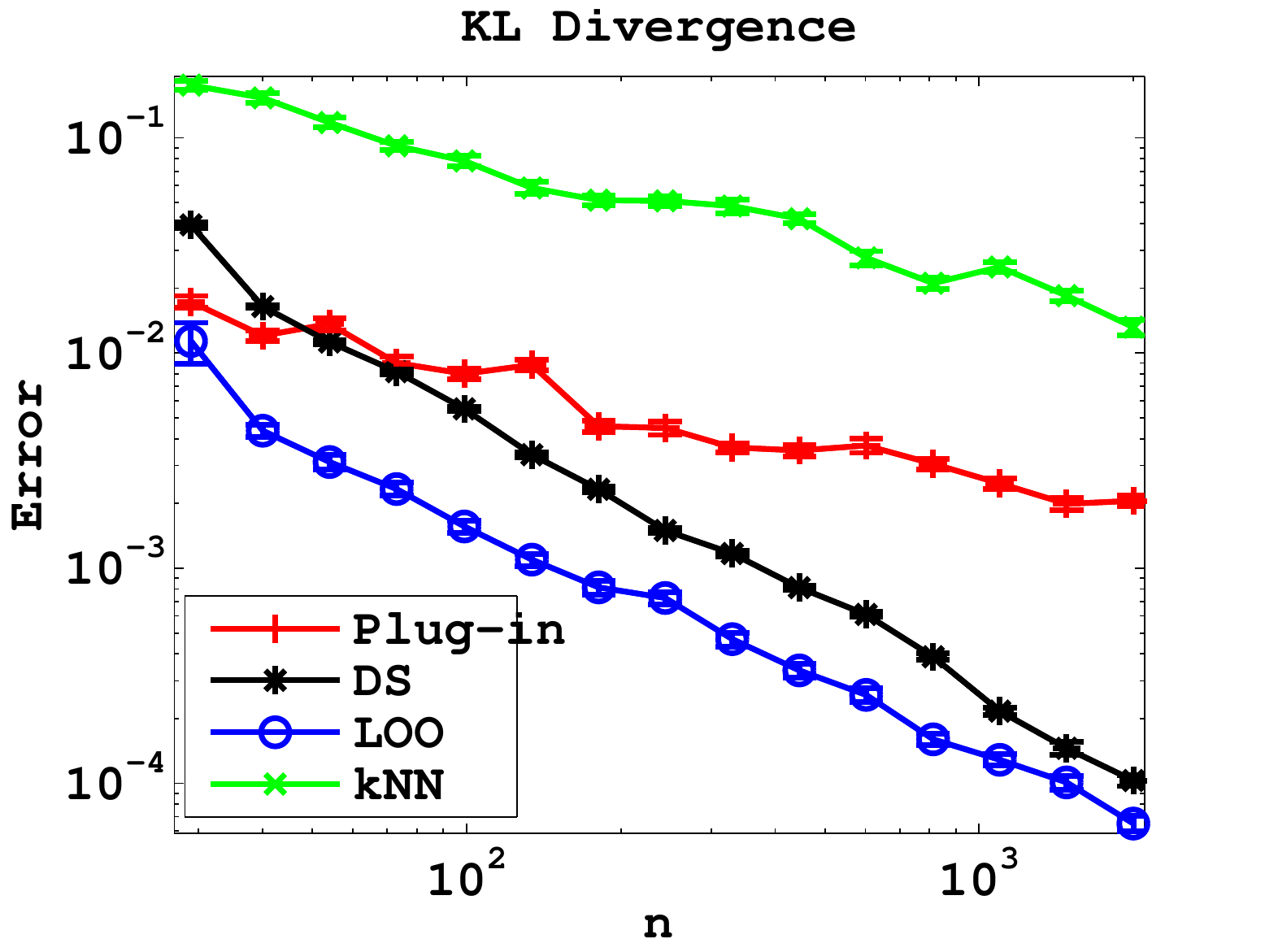}
  \label{fig:kl1D}
} \\[\imcaptionspace]
\subfigure{
  \includegraphics[width=\imarrwthree]{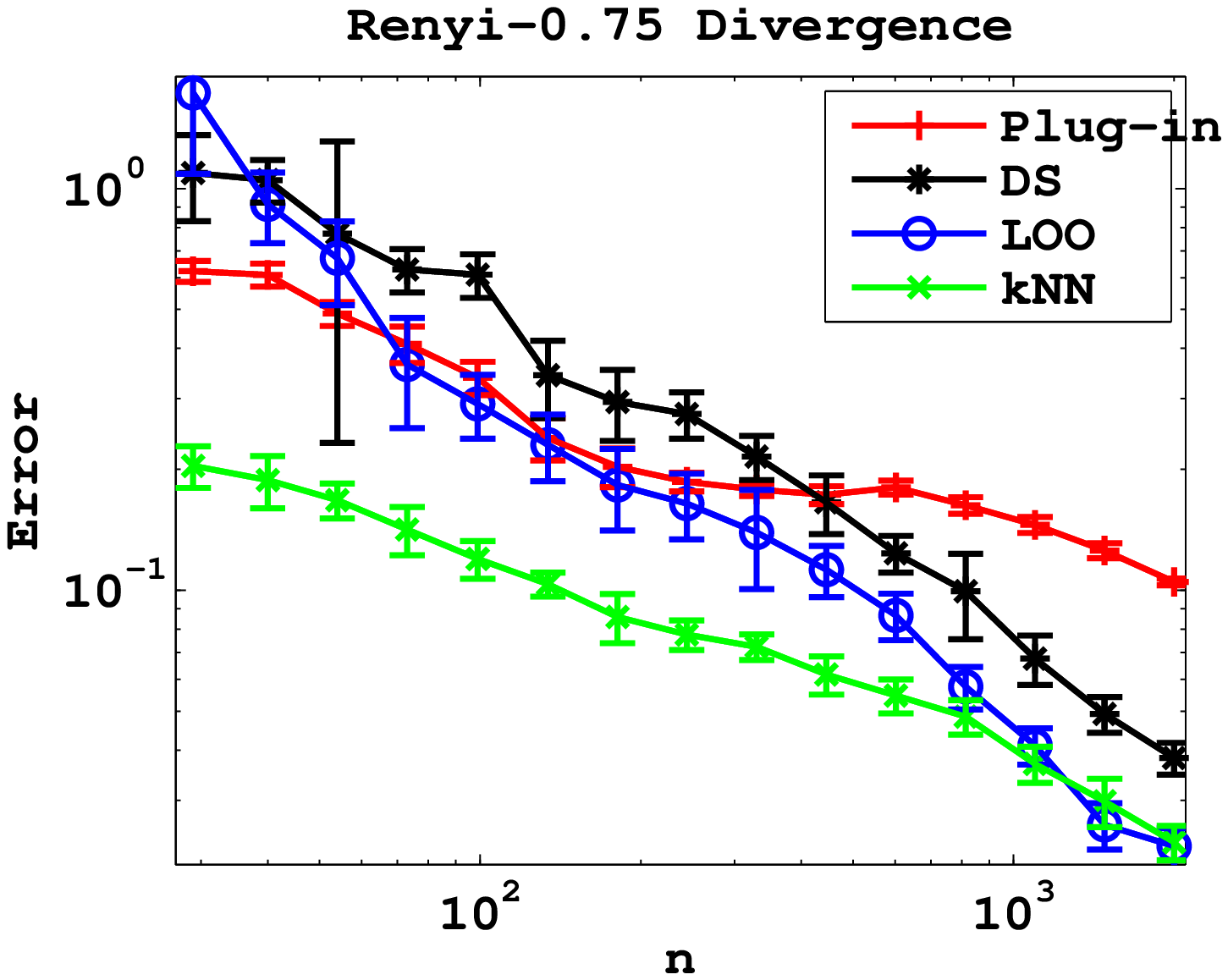} \hspace{\imhspthree}
  \label{fig:renyi1D}
}
\subfigure{
  \includegraphics[width=\imarrwthree]{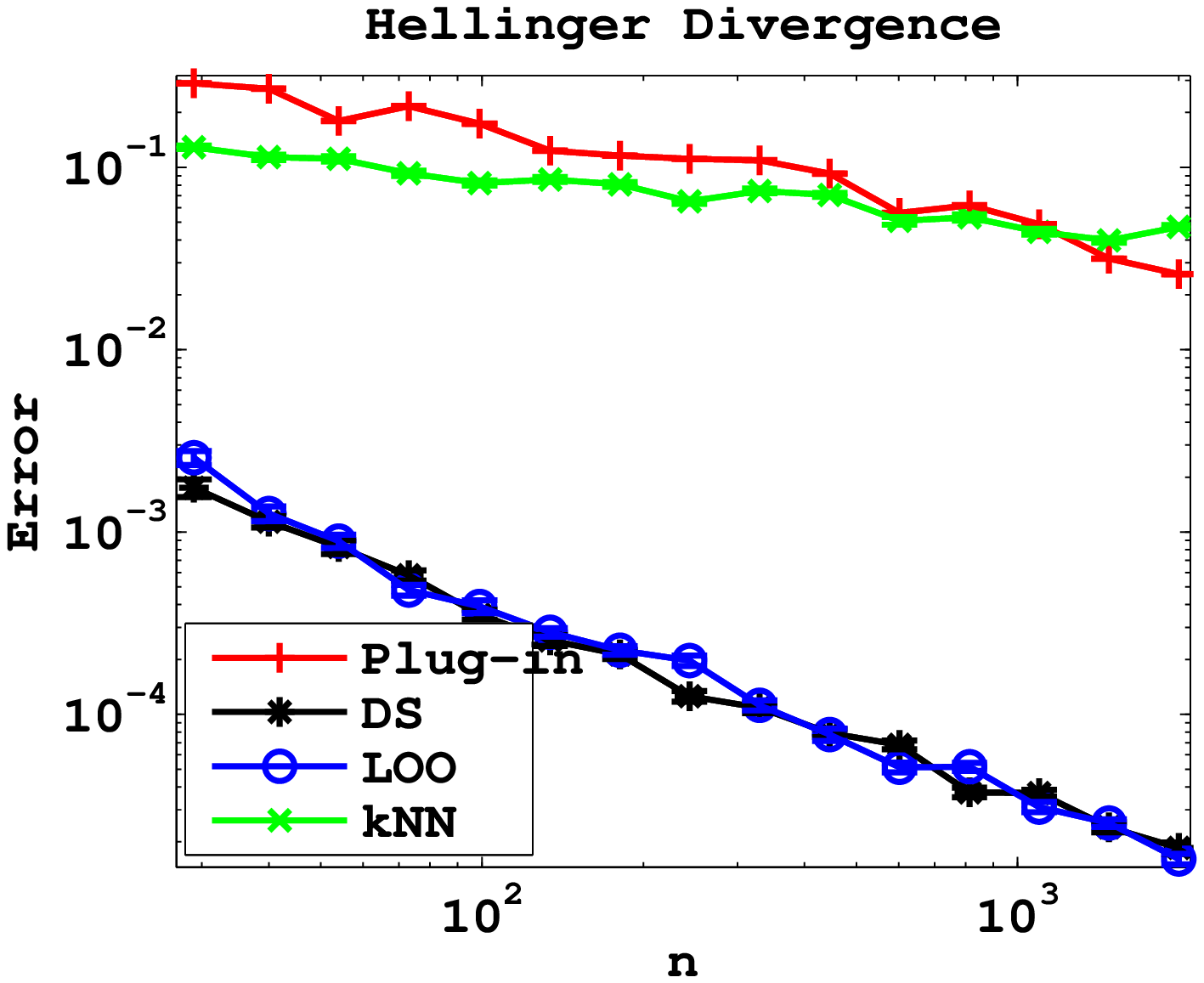} \hspace{\imhspthree}
  \label{fig:hellinger2D}
}
\subfigure{
  \includegraphics[width=\imarrwthree]{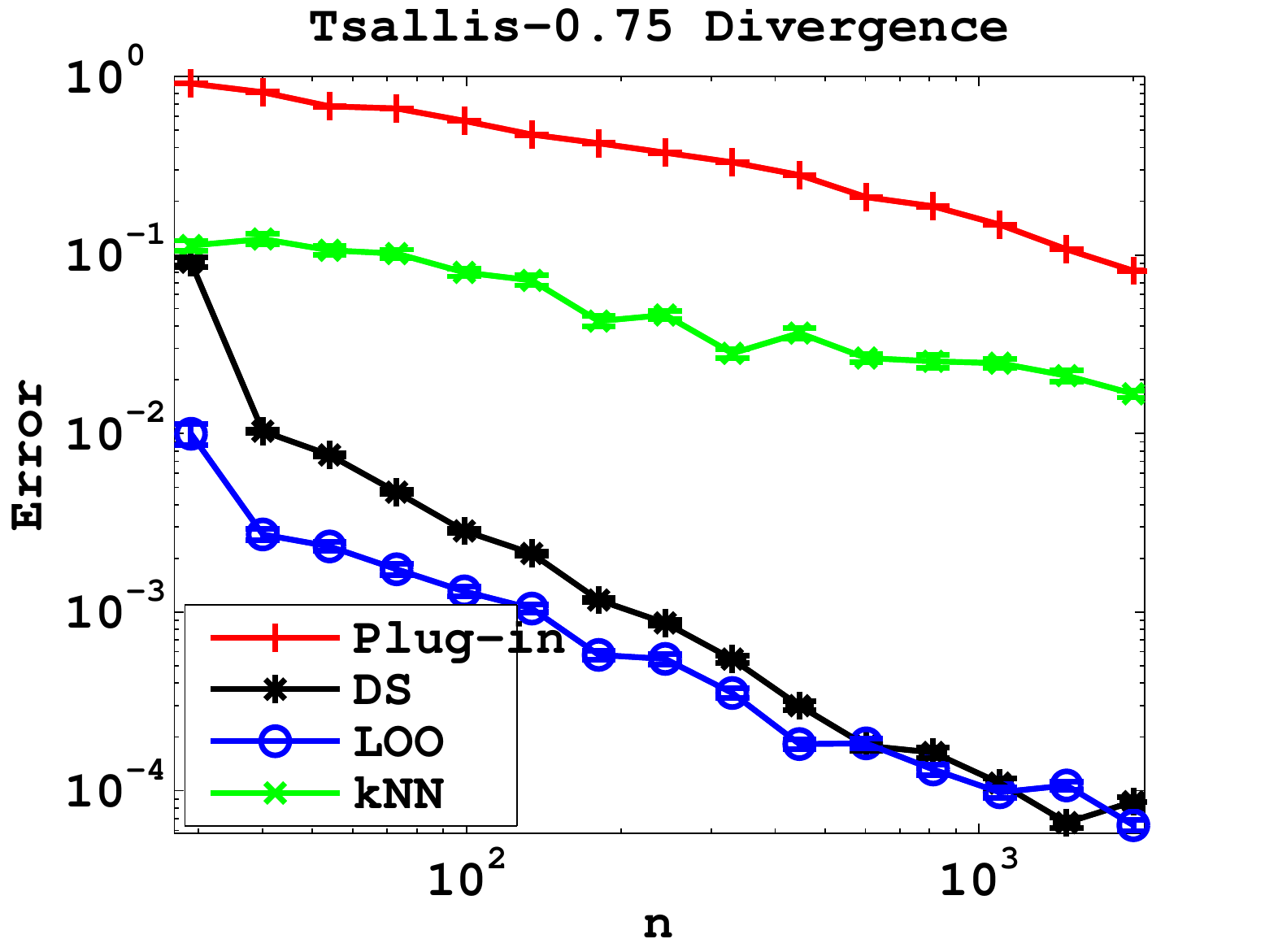}
  \label{fig:tsalllis2D}
}\\[\imcaptionspace]
\vspace{\imcaptionspace}
\caption[]{ \small 
Comparison of \ds/\loos estimators against alternatives on different 
functionals. The $y$-axis is the error $|\widehat{T} - T(f,g)|$ and the $x$-axis is the
number of samples. 
All curves were produced by averaging over $50$ experiments.
Some curves are slightly wiggly probably due to discretisation in hyperparameter
selection.
}
\label{fig:toyOne}
\vspace{\imtextspace}
\end{figure*}
}

\newcommand{\insertFigToyTwo}{
\begin{figure*}
\centering
\subfigure[]{
  \includegraphics[width=\imarrwthree]{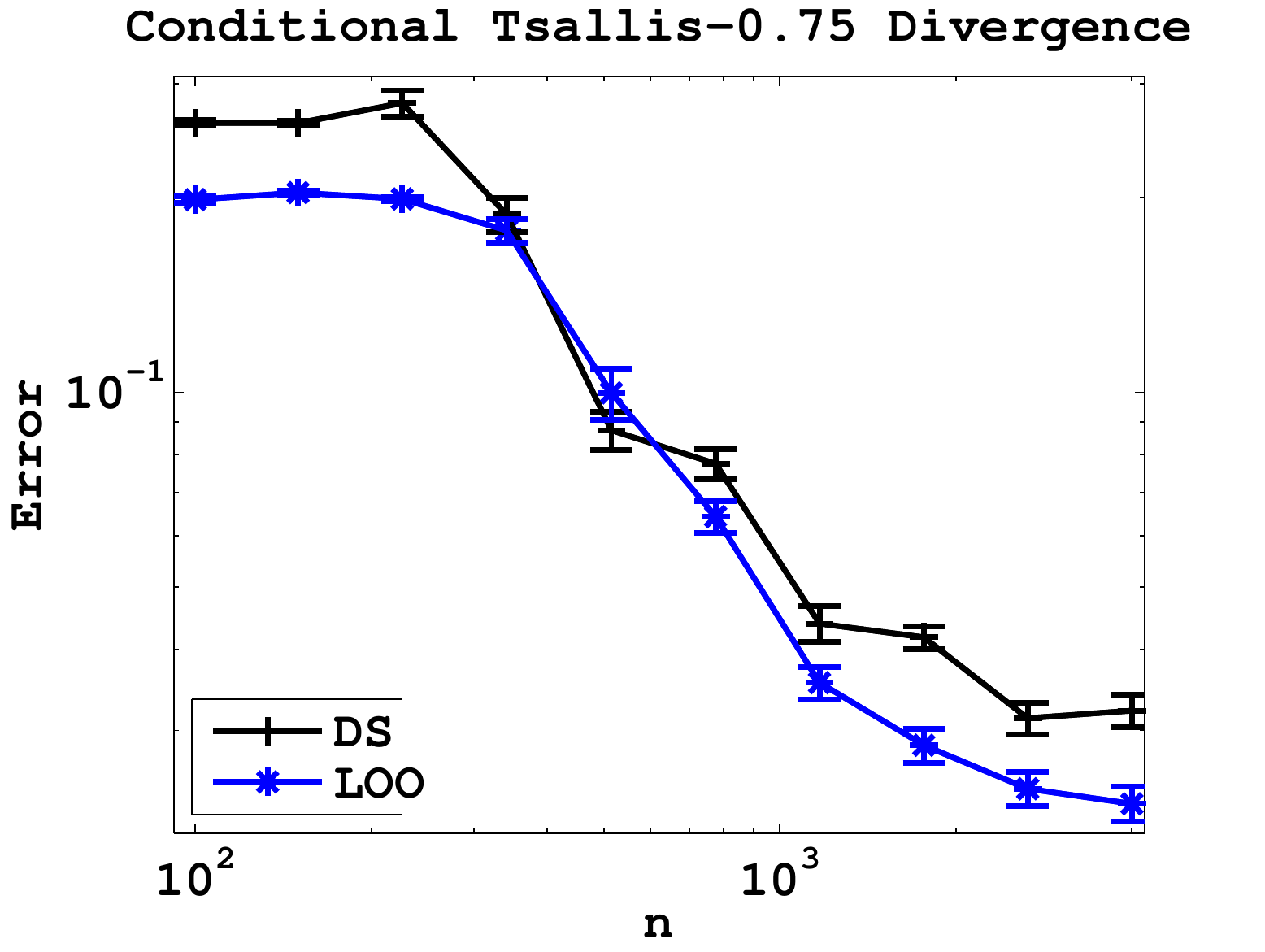} \hspace{\imhspthree}
  \label{fig:condTsallis4D}
}
\subfigure[]{
  \includegraphics[width=\imarrwthree]{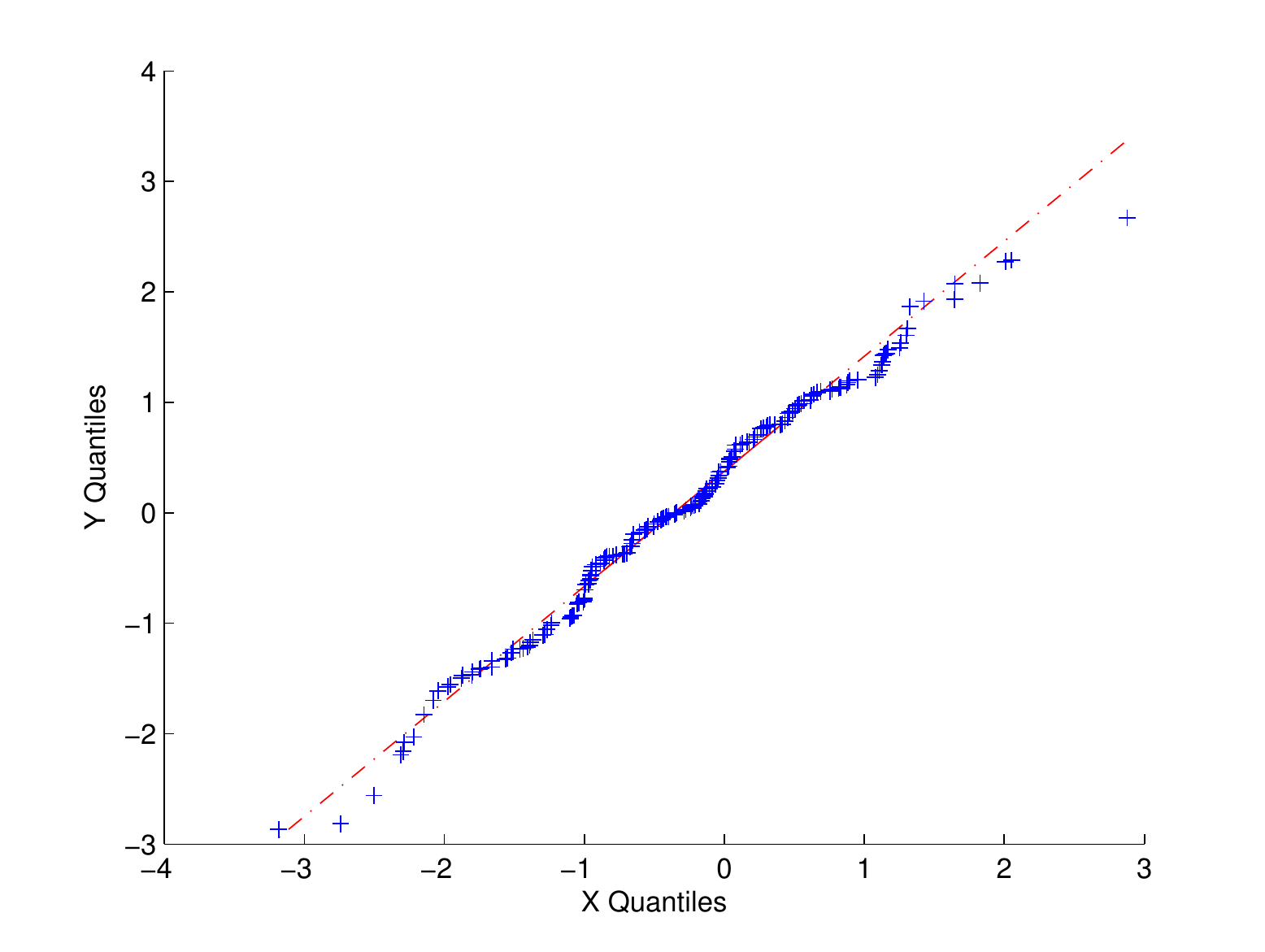} \hspace{\imhspthree}
  \label{fig:dsHellingerAN}
}
\subfigure[]{
  \includegraphics[width=\imarrwthree]{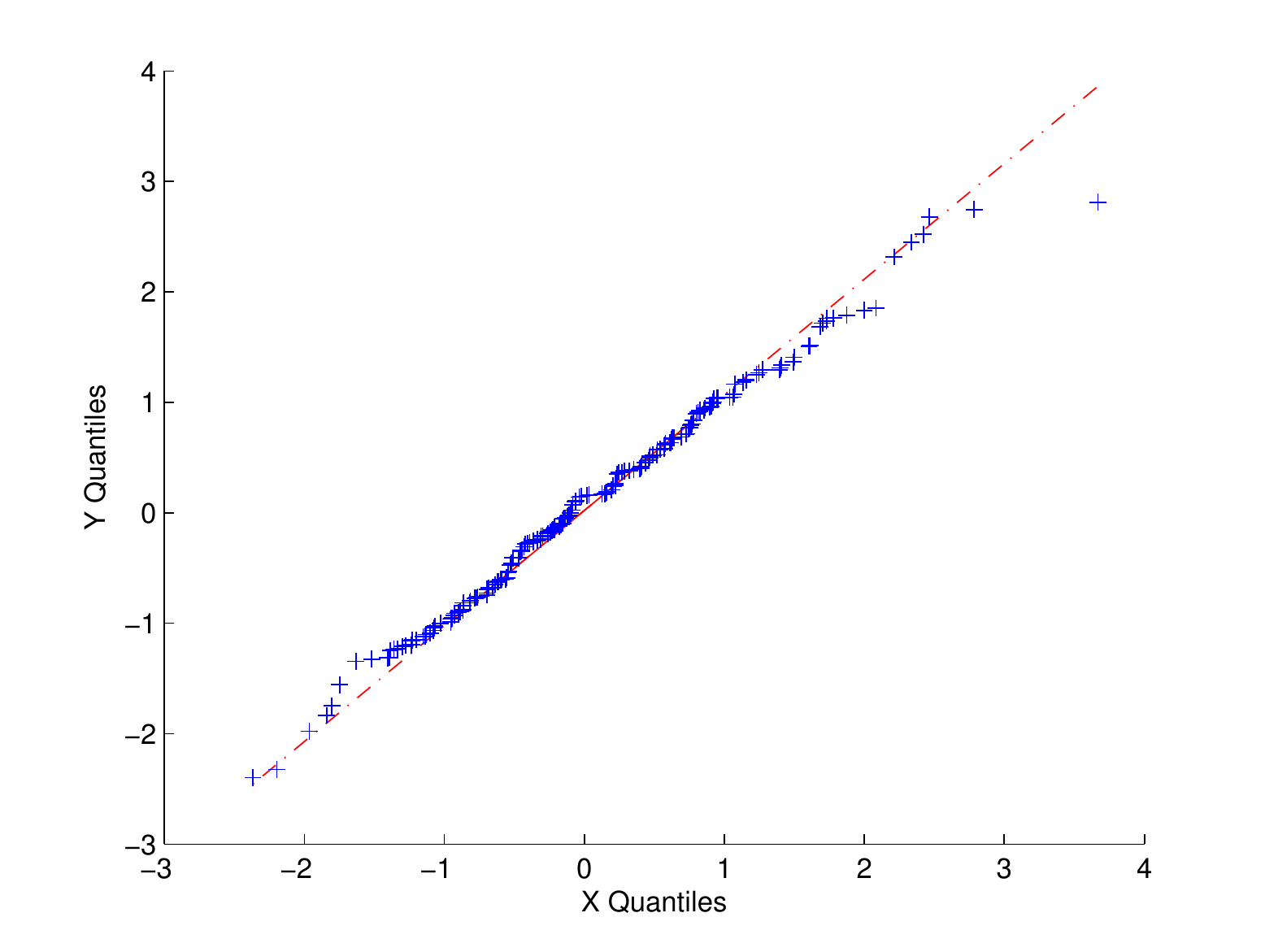} \hspace{\imhspthree}
  \label{fig:looHellingerAN}
}\\[\imcaptionspace]
\caption[]{ \small 
Fig~\subref{fig:condTsallis4D}: Comparison of the \loos vs \dss estimator on estimating the
Conditional Tsallis divergence in $4$ dimensions. Note that the plug-in estimator is
intractable due to numerical integration.
There are no other known estimators for the conditional tsallis divergence.
Figs~\subref{fig:dsHellingerAN},~\subref{fig:looHellingerAN}: QQ plots obtained using
$4000$ samples for Hellinger divergence estimation in $4$ dimensions using the \dss
and \loos estimators respectively.
}
\label{fig:toyTwo}
\vspace{\imtextspace}
\end{figure*}
}

\newcommand{\insertFigClustering}{
\begin{figure*}
\centering
\subfigure[]{
  \includegraphics[width=3.7in]{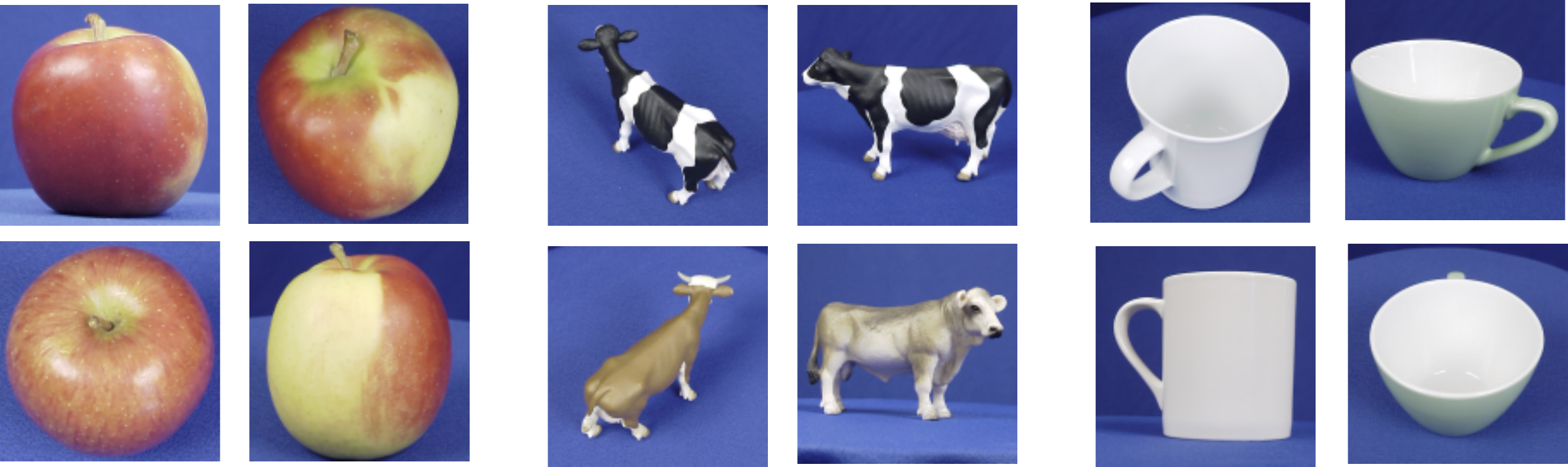} \hspace{0.1in} 
  \label{fig:clusImages}
}
\subfigure[]{
  \includegraphics[width=1.25in]{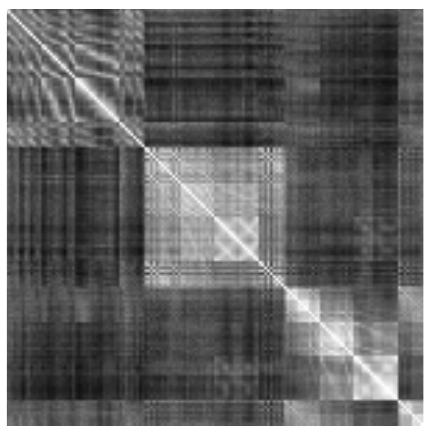}
  \label{fig:affinity}
}\\[\imcaptionspace]
\caption[]{ \small 
\subref{fig:clusImages} Some sample images from the three categories apples, cows and
cups. \subref{fig:affinity} The affinity matrix used in clustering.
}
\label{fig:clus}
\vspace{\imtextspace}
\end{figure*}
}

\begin{abstract}
We propose and analyze estimators for statistical functionals of one or more
distributions under nonparametric assumptions.
Our estimators are based on the theory of influence functions, which appear
in the semiparametric statistics literature.
We show that estimators based either on data-splitting or a leave-one-out technique 
enjoy fast rates of convergence and other favorable theoretical properties.
We apply this framework to derive estimators for several popular information 
theoretic quantities, and via empirical evaluation, show the advantage of this 
approach over existing estimators.
\end{abstract}

\if@twocolumn
\section{INTRODUCTION}
\else
\section{Introduction}
\fi
\label{sec:intro}

Entropies, divergences, and mutual informations are classical information-theoretic
quantities that play fundamental roles in statistics, machine learning,
and across the mathematical sciences.
In addition to their use as analytical tools, they arise in a variety of
applications including hypothesis testing, parameter estimation, feature selection, and optimal experimental
design. 
In many of these applications, it is important to \emph{estimate} these
functionals from data so that they can be used in downstream algorithmic or
scientific tasks. 
In this paper, we develop a recipe for estimating statistical
functionals of one or more nonparametric distributions based on the notion of
influence functions.

Entropy estimators are used in applications ranging from independent components 
analysis~\citep{learned2003ica}, intrinsic dimension 
estimation~\citep{carter10intrinsic} and several signal processing 
applications~\citep{hero02entropicGraphs}.
Divergence estimators are useful in statistical tasks such as two-sample testing.
Recently they have also gained popularity as they are used to measure (dis)-similarity between objects that are modeled as distributions, in what is known as the ``machine learning on distributions" framework~\citep{dhillon03information,poczos12divergence}.
Mutual information estimators have been used in in learning tree-structured Markov random fields~\citep{liu2012exponential}, feature selection~\citep{peng2005feature}, clustering~\citep{lewi2006real} and neuron classification~\citep{schneidman02neurons}.
In the parametric setting, conditional divergence and conditional mutual 
information estimators are 
used for conditional two sample testing or as 
building blocks for structure learning in graphical models.
Nonparametric estimators for these quantities
could potentially allow us to generalise several of these
algorithms to the nonparametric domain.
Our approach gives sample-efficient estimators for all these quantities
(and many others), which often outperfom the existing estimators both 
theoretically and empirically.

Our approach to estimating these functionals is based on post-hoc correction
of a preliminary estimator using the Von Mises Expansion
\cite{vandervaart98asymptotic,fernholz83vonmises}.
This idea has been used before in 
semiparametric statistics literature~\citep{birge95estimation,robins09quadraticvm}.
However, hitherto  most
studies are restricted to functionals of one distribution and have
focused on a ``data-split" approach which splits the samples for
density estimation and functional estimation.
While the data-split (\ds) estimator is known to achieve the parametric convergence
rate for sufficiently
smooth densities \cite{birge95estimation,laurent1996efficient}, in practical settings
splitting the data results in poor empirical performance.

In this paper we introduce the calculus of influence functions to the machine 
learning community and
considerably expand existing results by proposing a ``leave-one-out" (\loo) estimator
which makes efficient use of the data and has
better empirical performance than the DS technique. 
We also extend the framework of influence functions
to functionals of multiple distributions and develop both \dss and \loos 
estimators.
The main contributions of this paper are:
\begin{enumerate}[leftmargin=*]
\item We propose a \loos technique to estimate functionals of a single distribution 
with the same convergence rates as the \dss estimator. However, the
\loos estimator performs better empirically. 
\item We extend the framework to functionals of multiple distributions and
analyse their convergence.
Under sufficient smoothness both \dss and \loos estimators achieve the parametric 
rate and the \dss
estimator has a limiting normal distribution.
\item We prove a lower bound for estimating functionals of multiple distributions.
We use this to establish minimax optimality of the
 \dss and \loos estimators under sufficient smoothness.
\item We use the approach to construct and implement estimators for various entropy, 
divergence, mutual information quantities and their conditional versions. 
A subset of these functionals are
listed in Table~\ref{tb:functionalDefns}. For several functionals, \emph{these are the only
known estimators}.
Our software is publicly available at
\texttt{github.com/kirthevasank/if-estimators}.
\item We compare our estimators against several other approaches in simulation.
Despite the generality of our approach, our estimators are
competitive with and in many cases superior to existing specialized approaches for
specific functionals. We also demonstrate how our estimators can be used in machine
learning applications via an image clustering task.
\end{enumerate}

Our focus on information theoretic quantities is due to their
relevance in machine learning applications, rather than a limitation
of our approach. Indeed our techniques apply to any smooth functional.

\label{sec:relatedwork}

\textbf{History:}
We provide a brief history of the post-hoc correction technique and influence functions.
We defer a detailed discussion of other approaches to estimating 
functionals to Section~\ref{sec:comparison}.
To our knowledge, the first paper using a post-hoc correction 
estimator for  was that of~\citet{bickel1988estimating}.  
The line of work
following this paper analyzed integral functionals of a single one
dimensional density of the form $\int \nu(p)$
~\citep{bickel1988estimating,birge95estimation,
laurent1996efficient,kerkyacharian1996estimating}.  A recent paper
by \citet{krishnamurthy14renyi} also extends this line to functionals
of multiple densities, but only considers polynomial functionals of
the form $\int p^\alpha q^\beta$ for densities $p$ and $q$.  
Moreover, all these works use data splitting.
Our work builds on these by extending to a more general class of
functionals and we propose the superior \loos estimator.

A fundamental quantity in the design of our estimators is
the \emph{influence function}, which appears both in
robust and semiparametric
statistics.  Indeed, our work is inspired 
by that of~\citet{robins09quadraticvm} and~\citet{emery98probstat} 
who propose a (data-split)
influence-function based estimator for functionals of a single
distribution. 
Their analysis for nonparametric problems rely
on ideas from semiparametric statistics: they define
influence functions for parametric models and then analyze estimators
by looking at all parametric submodels through the true parameter.

\section{Preliminaries}
\label{sec:prelims}

Let $\Xcal$ be a compact metric space equipped with a measure $\mu$, e.g. the 
Lebesgue measure.
Let $P$ and $Q$ be measures over $\Xcal$ that are absolutely continuous w.r.t $\mu$.
Let $p,q \in L_2(\Xcal)$ be the Radon-Nikodym derivatives with respect to $\mu$.
We focus on estimating functionals of the form:
\begin{align}
T(P) &= T(p) = \phi\left(\int\nu(p)d\mu\right) \qquad \textrm{or} \qquad
T(P,Q) = T(p,q) = \phi\left(\int \nu(p,q) d\mu\right),
\end{align}
where $\phi, \nu$ are real valued Lipschitz functions that twice differentiable. 
Our framework permits more general
functionals -- e.g. functionals based on the conditional densities --
but we will focus on this form for ease of exposition.
To facilitate presentation of the main definitions, it is easiest to work with 
functionals of one distribution $T(P)$. 
Define $\Mcal$ to be the set of all measures that are absolutely 
continuous w.r.t $\mu$, whose Radon-Nikodym derivatives belong to $L_2(\Xcal)$. 

Central to our development is the Von Mises expansion (VME), which is the 
distributional analog of the Taylor expansion.
For this we introduce the \gateaux derivative which imposes a notion of
differentiability in topological spaces. We then introduce the \emph{influence
function}. \\[\thmparaspacing]

\begin{definition}
The map $T':\Mcal \rightarrow \RR$ where 
$
T'(H;P) = \partialfrac{t}{T(P + tH )}|_{t=0}
$
is called the \textbf{\gateaux
derivative} at $P$ if the derivative exists and is linear and continuous in $H$.
We say $T$ is \gateaux differentiable at $P$ if $T'$ exists at $P$.
\\[\thmparaspacing]
\end{definition}


\begin{definition}
Let $T$ be \gateaux differentiable at $P$.
A function $\psi(\cdot; P) : \Xcal \rightarrow \RR$ which satisfies
$
T'(Q-P; P) = \int \psi(x; P) \ud Q(x),
$
is the \textbf{influence function} of $T$ w.r.t the distribution $P$. 
\label{def:infFun}
\end{definition}

The existence and uniqueness of the influence function is guaranteed by the
Riesz representation theorem, since the domain of $T$ is a bijection of 
$L_2(\Xcal)$ and consequently a Hilbert space.
The classical work of~\citet{fernholz83vonmises} defines the influence function 
in terms of the \gateaux derivative by,
\begin{equation}
\psi(x, P) = T'(\delta_x - P, P) = \partialfracat{t}{T((1-t)P + t\delta_x )}{0},
\label{eqn:infFunGateaux}
\end{equation}
where $\delta_x$ is the dirac delta function at $x$.
While our functionals are defined only on non-atomic distributions, we can still
use~\eqref{eqn:infFunGateaux} to compute the influence function. 
The function computed this way can be shown to satisfy
Definition~\ref{def:infFun}.

Based on the above, the first order VME is,
\ifthenelse{\boolean{istwocolumn}} 
  {
  \begin{align*}
  T(Q) &= T(P) + T'(Q-P; P) + R_2(P,Q)  \numberthis 
      \label{eqn:oneDistroVME}\\
       &= T(P) + \int \psi(x; P) \ud Q(x) + R_2(P,Q),
  \end{align*}
  }
  {
  \begin{align}
  T(Q) = T(P) + T'(Q-P; P) + R_2(P,Q) 
       = T(P) + \int \psi(x; P) \ud Q(x) + R_2(P,Q),
    \label{eqn:oneDistroVME}
  \end{align}
  }
where $R_2$ is the second order remainder.
\gateaux differentiability alone will not be
sufficient for our purposes. In what follows, we will assign 
$Q  \rightarrow F$ 
and $P\rightarrow \widehat{F}$, where $F$, $\widehat{F}$ are the
true and estimated distributions. 
We would like to bound the remainder in
terms of a distance between $F$ and $\widehat{F}$. 
By taking the domain of $T$ to be only measures with continuous densities,
we can control $R_2$ using the $L_2$ metric of the
densities.
This essentially  means  that our functionals satisfy a stronger form of
differentiability called \frechet differentiability 
\citep{vandervaart98asymptotic, fernholz83vonmises} in the $L_2$ metric.
Consequently, we can write all derivatives in terms of the
densities, and the VME reduces to a functional Taylor expansion on the densities 
(Lemmas~\ref{lem:vmeTaylorEqOneDistro},~\ref{lem:vmeTaylorEqTwoDistro} 
in Appendix~\ref{sec:appAncillaryResults}):
\ifthenelse{\boolean{istwocolumn}}
  {
  \begin{align*}
  T(q) &= T(p) + \phi'\left(\int \nu(p)\right)\int(q-p) \nu'(p) \\
  &\hspace{0.8in}+  R_2(p,q) \\ 
    &= T(p) + \int \psi(x; p) q(x) \ud x +  \bigO(\|p-q\|_2^2).
    \numberthis
  \label{eqn:functaylor}
  \end{align*}
  }
  {
  \begin{align*}
  T(q) &= T(p) + \phi'\left(\int \nu(p)\right)\int(q-p) \nu'(p) + 
      R_2(p,q) \\ 
    &= T(p) + \int \psi(x; p) q(x) \ud x +  \bigO(\|p-q\|_2^2).
    \numberthis
  \label{eqn:functaylor}
  \end{align*}
  }
This expansion will be the basis for our estimators. 

These ideas generalize to functionals of multiple distributions and to settings where the functional involves quantities other than the density.
A functional $T(P,Q)$ of two distributions has two \gateaux derivatives, $T'_i(\cdot; P,Q)$ for $i=1,2$ formed by perturbing the $i$th argument with the other fixed. 
The influence functions $\psi_1,\psi_2$ satisfy, $\forall P_1,P_2\in\Mcal$, 
\begingroup
\allowdisplaybreaks
\begin{align*}
&T_1'(Q_1-P_1; P_1, P_2) = 
\partialfracat{t}{T(P_1+t(Q_1-P_1), P_2)}{0} = 
\int \psi_1(u; P_1, P_2) \ud Q_1(u). 
\numberthis \label{eqn:defInfFun} \\
&T_2'(Q_2-P_2; P_1, P_2) = 
\partialfracat{t}{T(P_1,P_2 +t(Q_2-P_2))}{0} = 
\int \psi_2(u; P_1, P_2) \ud Q_2(u).
\end{align*}
\endgroup
The VME can be written as,
\begin{align*}
T(q_1, q_2) &= T(p_1, p_2) + \int \psi_1(x; p_1, p_2) q_1(x) \ud x + 
  \int \psi_2(x; p_1, p_2) q_2(x) \ud x  \\
  &\hspace{0.2in}+  \bigO(\|p_1-q_1\|_2^2) + \bigO(\|p_2-q_2\|_2^2).
  \numberthis
\label{eqn:functaylortwoD}
\end{align*}





\section{Estimating Functionals}
\label{sec:estimation}

First consider estimating a functional of a single distribution,
$T(f)= \phi(\int \nu(f) d\mu)$ 
from samples $\Xonetwo \sim f$. 
Using the VME~\eqref{eqn:functaylor}, \citet{emery98probstat} 
and~\citet{robins09quadraticvm} 
suggested a natural estimator.
If we use half of the data $\Xone$ to construct a density estimate
$\fhatone = \fhatone(\Xone)$, then by~\eqref{eqn:functaylor}:
\begin{align*}
T(f) - T(\fhatone)  = \int \psi(x; \fhatone)f(x)d\mu + \bigO(\|f - \fhatone\|_2^2).
\end{align*}
Since the influence function does not depend on the unknown distribution $F$,
the first term on the right hand side is simply an expectation of $\psi(X;\fhatone)$
at $F$.
We use the second half of the data to estimate this expectation with its sample mean.
This leads to the following preliminary estimator:
\begin{equation}
 \Tfdsone = T(\fhatone) + \frac{1}{n/2}\nsumsechalf \psi(X_i; \fhatone).
\label{eqn:dsEstimator}
\end{equation}
We can similarly construct an estimator $\Tfdstwo$ by using $\Xtwo$ 
for density estimation and $\Xone$ for averaging. Our final estimator is
obtained via the average $\Tfds =(\Tfdsone+\Tfdstwo)/2$. 
In what follows, we shall refer to this estimator as the Data-Split (\ds) estimator.

The rate of convergence of this estimator is determined by the error in the VME 
$\bigO(\|f-\fhatone\|_2^2)$ and the $n^{-1/2}$ rate for estimating an expectation.
Lower bounds from several literature 
\citep{laurent1996efficient,birge95estimation} 
confirm minimax optimality of the \dss estimator when $f$ is sufficiently
smooth. The data splitting trick is commonly used in several other works
\cite{birge95estimation, laurent1996efficient, krishnamurthy14renyi}.
The analysis of \dss estimators is straightforward as the rate directly follows from the 
Cauchy-Schwarz inequality. 
While in theory, \dss estimators enjoy good rates of convergence,
from a practical stand point, the data splitting 
is unsatisfying since using
only half the data each for estimation and averaging invariably decreases the accuracy.

As an alternative, we propose a Leave-One-Out (\loo) version of the above estimator which
makes more efficient use of the data,
\begin{equation}
\Tf = \frac{1}{n}\sum_{i=1}^n  T(\fhatmi) + \psi(X_i; \fhatmi).
\label{eqn:looEstimator}
\end{equation}
where $\fhatmi$ is the kernel density estimate using all the samples $\Xn$
except for $X_i$.
Theoretically, we prove that the \loos Estimator achieves the same rate of convergence as 
the \dss estimator but  emprically it performs much better.

We can extend this method for functionals of two distributions.
Akin to the one distribution case, we propose the following \dss and \loos versions.
\begingroup
\allowdisplaybreaks
\begin{align*}
  \Tfdsone &= T(\fhatone, \ghatone) +
    \frac{1}{n/2} \nsumsechalf \psi_f(X_i; \fhatone, \ghatone)  
    + \frac{1}{m/2} \msumsechalf \psi_g(Y_j; \fhatone, \ghatone). \numberthis 
  \label{eqn:dsEstimatorTwoD} \\
\Tf &= \frac{1}{\max(n,m)}\sum_{i=1}^{\max(n,m)} 
  T(\fhatmi, \ghatmi) + \psif(X_i;\fhatmi, \ghatmi) + 
  \psig(Y_i;\fhatmi, \ghatmi). \numberthis
  \label{eqn:looEstimatorTwoD}
\end{align*}
\endgroup
Here, $\ghatone,\ghatmi$ are defined similar to $\fhatone,\fhatmi$. 
For the \dss estimator we swap the samples to compute $\Tfdstwo$ and then
average.
For the \loos estimator, if $n>m$ we cycle through the points $\Ym$ until we have
summed over all $\Xn$ or vice versa.
Note that $\Tf$ is asymmetric when $n\neq m$.
A seemingly natural alternative would be to sum over all $nm$ pairings of
$X_i$'s and $Y_j$'s. However, the latter approach is more computationally
burdensome. Moreover, a straightforward modification of our analysis in
Appendix~\ref{sec:looEstimatorTwoD} shows that both estimators have the same
rate of convergence if $n$ and $m$ are of the same order.

\textbf{Examples: }
We demonstrate the generality 
of our framework by presenting estimators for several 
entropies, divergences and mutual informations and their
conditional versions in Table~\ref{tb:functionalDefns}.
For several functionals in the table, \emph{these are the first estimators proposed}.
We hope that this table will serve as a good reference for practitioners.
For several functionals (e.g. conditional and unconditional \renyi divergence,
conditional \tsallis mutual information and more) the estimators are not listed only
because the expressions are too long to fit into the table. Our software implements
a total of $17$ functionals which include all the estimators in the table.
In Appendix~\ref{sec:workedExample} we illustrate how to apply our framework
to derive an estimator for any functional via an example.

As will be discussed in Section~\ref{sec:comparison},
when compared to other alternatives, our technique has several favourable properties.
Computationally, the complexity of our method is $O(n^2)$ when compared to $O(n^3)$
for some other methods and for several functionals we do not require numeric
integration.
Additionally, unlike most other methods, we do not require any tuning of
hyperparameters.

\newcommand{\twolinecell}[2][c]{%
  \begin{tabular}[#1]{@{}c@{}}#2\end{tabular}}
 \newcommand{\fourlinecell}[2][c]{%
   \begin{tabular}[#1]{@{}c@{}c@{}c@{}}#2\end{tabular}}
\newcommand{\pxayb}{\int \px^\alpha \py^{1-\alpha}}
\newcommand{\pxzayzb}{\int \pxz^\alpha \pyz^{1-\alpha}}
\begin{table*}[htbp]
\begin{center}
\small{
\begin{tabular}{|c|c|}
\hline
\textbf{Functional} & \textbf{\loos Estimator} \\ \hline
\twolinecell{\tsallis Entropy \\ $\frac{1}{\alpha-1}\left(1 - \int p^\alpha\right)$} & 
$\frac{1}{1-\alpha} +  
 \frac{1}{n}\sum_i \int \phatmi^\alpha -\frac{\alpha}{\alpha-1} \phatmi^{\alpha-1}(X_i) $ \\ \hline

\twolinecell{\renyi Entropy \\ $\frac{-1}{\alpha-1}\log \int p^\alpha$} &
  $ \frac{\alpha}{\alpha-1} + \frac{1}{n}\sum_i
  \frac{-1}{\alpha-1}\log \int \phatmi^\alpha -
   \phatmi^{\alpha-1}(X_i) + $
   \\ \hline

\twolinecell{Shannon Entropy \\ $- \int p \log p$} &
  $ -\frac{1}{n} \sum_i \log \phatmi(X_i)$ 
   \\ \hline

\twolinecell{$\ltwotwo$ Divergence\\ $  \int (p_X - p_Y)^2$} &
    $\frac{2}{n}\sum_i \pxhatmi(X_i) - \pyhat(X_i) - \int (\pxhatmi - \pyhat)^2 +
      \frac{2}{m}\sum_j \pxhat(Y_j) - \pyhatmj(Y_j) $
   \\ \hline

\twolinecell{Hellinger Divergence \\ $ 2-2\int \px^{1/2}\py^{1/2} $} &
    $2- \frac{1}{n}\sum_i \pxhatmi^{-1/2}(X_i)\pyhat^{1/2}(X_i) 
       - \frac{1}{m}\sum_j \pxhat^{1/2}(Y_j)\pyhatmj^{-1/2}(Y_j) $
 \\ \hline

\twolinecell{Chi-Squared Divergence \\ $\int \frac{(\px-\py)^2}{\px}$}&
  $ -1 + \frac{1}{n}\sum_i \frac{\pyhat^2(X_i)}{\pxhatmi^2(X_i)} +
      2\frac{1}{m}\sum_j \frac{\pyhatmj(Y_j)}{\pxhat(Y_j)} $\\ \hline

\twolinecell{$f$-Divergence \\ $\int \phi ( \frac{\px}{\py} ) \py $}&
  $ \frac{1}{n}\sum_i\phi' \left(\frac{\pxhatmi(X_i)}{\pyhat(X_i)}\right)
  + \frac{1}{m} \sum_j \left(\phi\left(\frac{\pyhatmj(Y_j)}{\pxhat(Y_j)}\right)
  - \frac{\pxhat(Y_j)}{\pyhatmj(Y_j)}\phi\left(\frac{\pxhat(Y_j)}{\pyhatmj(Y_j)}\right) \right) 
   $\\ \hline


\twolinecell{\tsallis Divergence \\ $ \frac{1}{\alpha-1}\left(
      \int \px^\alpha \py^{1-\alpha} -1 \right) $} & 
  $ \frac{1}{1-\alpha} + \frac{\alpha}{\alpha-1}\frac{1}{n}\sum_i
  \left(\frac{\pxhatmi(X_i)}{\pyhat(X_i)}\right)^{\alpha-1} -
  \frac{1}{m}\sum_j \left( \frac{\pxhat(Y_j)}{\pyhatmj(Y_j)}\right)^{\alpha} $\\ \hline


\twolinecell{KL divergence \\ $ \int p_{X}\log\frac{p_{X}}{p_{Y}} $} &
    $ 1 + \frac{1}{n}\sum_i \log \frac{\pxhatmi(X_i)}{\pyhat(X_i)} - 
      \frac{1}{m}\sum_j \frac{\pxhat(Y_j)}{\pyhatmj(Y_j)} $ \\ \hline

\twolinecell{Conditional-\tsallis divergence \\ $ \int  p_Z\frac{1}{\alpha-1}\left(\int p_{X|Z}^{\alpha}
    p_{Y|Z}^{1-\alpha} -1\right) $} &
  $ \frac{1}{1 - \alpha} + \frac{\alpha}{\alpha-1}
  \frac{1}{n}\sum_i \left(\frac{\pxzhatmi(V_i)}{\pyzhat(V_i)}\right)^{\alpha-1}
   - \frac{1}{m}\sum_{j} \left(\frac{\pxzhat(W_j)}{\pyzhatmj(W_j)}\right)^{\alpha} $
 \\ \hline

\twolinecell{Conditional-KL divergence \\ $ \int p_Z\int p_{X|Z}\log\frac{p_{X|Z}}{p_{Y|Z}} $} &
$ 1 + \frac{1}{n}\sum_i \log \frac{\pxzhatmi(V_i)}{\pyzhat(V_i)} - 
  \frac{1}{m}\sum_j \frac{\pxzhat(W_j)}{\pyzhatmj(W_j)} $ \\ \hline

%

\twolinecell{Shannon Mutual Information \\
   $ \int p_{XY}\log\frac{p_{XY}}{p_{X}p_Y} $} &
    $ \frac{1}{n} \sum_i \log \pxyhatmi(X_i, Y_i) -\log \pxhatmi(X_i) 
        -\log \pyhatmi(Y_i)  $ \\ \hline

\twolinecell{Conditional \tsallis MI \\
  $\int  p_Z \frac{1}{\alpha-1}\left(\int p_{X,Y|Z}^{\alpha}
      p_{X|Z}^{1-\alpha}p_{Y|Z}^{1-\alpha} -1\right)$} &
\fourlinecell{$ 
      \frac{1}{1-\alpha} + \frac{1}{\alpha-1}\frac{1}{n}\sum_i 
      \alpha \left(\frac{\pxyzhatmi(X_i,Y_i,Z_i)\pzhat(Z_i)}
        {\pxzhatmi(X_i,Z_i)\pyzhatmi(Y_i,Z_i)}\right)^{\alpha-1}$ \\
      {\scriptsize
          $-  (1-\alpha) \frac{1}{\alpha-1}\frac{1}{n}\sum_i\pzhat^{\alpha-2}(Z_i) \int\pxyzhatmi^\alpha(\cdot,\cdot,Z_i)
            \pxzhatmi^{1-\alpha}(\cdot,Z_i) $} \\
  { \scriptsize
    $ +\frac{1}{\alpha-1}\frac{1}{n}\sum_i
       (1-\alpha)\pxzhatmi^{-\alpha}(X_i,Z_i)\pzhat^{1-\alpha}(Z_i) 
        \int\pxyzhatmi^\alpha(X_i,\cdot,Z_i)\pyzhatmi^{1-\alpha}(\cdot,Z_i) 
    $}
       \\
    {\scriptsize
    $+ \frac{1}{\alpha-1}\frac{1}{n}\sum_i
       (1-\alpha)\pyzhatmi^{-\alpha}(Y_i,Z_i)\pzhat^{\alpha -1}(Z_i)
        \int \pxyzhatmi^\alpha(\cdot,Y_i,Z_i)\pxzhatmi^{1-\alpha}(\cdot,\cdot) 
    $} }
  \\ \hline
\end{tabular}
}
\vspace{-0.1in}
\caption{ 
Definitions of functionals and the corresponding estimators. 
Here $p_{X|Z}, \pxz$ etc. are conditional and joint distributions.
For the conditional divergences we take $V_i = (X_i, Z_{1i})$, $W_j = (Y_j,
Z_{2j})$ to be the samples from $\pxz, \pyz$ respectively.
For the mutual informations we have samples $(X_i, Y_i)\sim \pxy$ and for
the conditional versions we have $(X_i, Y_i, Z_i)\sim \pxyz$.
}
\label{tb:functionalDefns}
\end{center}
\end{table*}

\section{Analysis}
\label{sec:analysis}

Some smoothness assumptions on the densities are warranted to make estimation
tractable.
We use the \holder class, which is now standard in 
nonparametrics literature. \\

\begin{definition}[\holder Class]
Let $\Xcal \subset \RR^d$ be a compact space.
For any $r = (r_1, \dots, r_d), r_i \in \NN$, define $|r| = \sum_i r_i$ and
$D^r = \frac{\partial^{|r|}}{\partial x_1^{r_1} \dots \partial x_d^{r_d}}$. The
\holder class \textbf{$\Sigma(s, L)$} is the set of functions on $L_2(\Xcal)$ satisfying,
\[
  |D^r f(x) - D^r f(y)| \leq L \|x - y\|^{s - r}, \;\;
\]
for all $r$ s.t. $|r| \leq \floor{s}$ and for all $x, y \in \Xcal$.
\end{definition}
Moreover, define the Bounded \holder Class  $\Sigma(s, L, B', B)$ to be
$\cbr{f\in \Sigma(s,L): B' < f < B}$. 
Note that large $s$ implies higher smoothness.
Given $n$ samples $\Xn$ from a $d$-dimensional density $f$, the kernel density 
estimator (KDE) with bandwidth $h$ is $\fhat(t) = 1/(nh^d) \sum_{i=1}^n 
K\left(\frac{t-X_i}{h} \right)$.
Here $K:\RR^d \rightarrow \RR$ is a smoothing kernel \cite{tsybakov08nonparametric}.
When $f \in \Sigma(s,L)$, by
selecting $h \in \Theta(n^{\frac{-1}{2s+d}})$ 
the KDE achieves the minimax rate of $\bigOp(n^{\frac{-2s}{2s+d}})$ in mean squared
error. 
Further, if $f$ is in the bounded \holder class $\Sigma(s, L, B',B)$ one can 
truncate the KDE from below at $B'$ and from above at $B$ 
and achieve the same convergence rate~\citep{birge95estimation}. 
In our analysis, the density estimators $\fhatone, \fhatmi, \ghatone, \ghatmi$ 
are formed by either a KDE or a 
truncated KDE, and we will make use of these results. 

We will also need the following regularity condition 
on the influence
function. This is satisfied for smooth functionals including those
in Table~\ref{tb:functionalDefns}. We demonstrate this in our example in
Appendix~\ref{sec:workedExample}.
\\[\thmparaspacing]
%

\begin{assumption}
For a functional $T(f)$, the influence function $\psi$ satisfies,
\[
\EE\big[(\psi(X;f') - \psi(X;f))^2\big] 
\in \bigO( \|f-f'\|^2 ) \;\;\textrm{ as }\;\;
\|f-f'\|^2 \rightarrow 0.
\]
For a functional $T(f,g)$ of two distributions, the influence functions $\psi_f,
\psi_g$ satisfy,
\begin{align*}
&\EE_f\Big[(\psi_f(X;f', g') - \psi_f(X;f,g))^2\Big] 
\in \bigO( \|f-f'\|^2 + \|g-g'\|^2 )
 \;\;\textrm{ as }\;\;
\|f-f'\|^2, \|g-g'\|^2 \rightarrow 0. \\
&\EE_g\Big[(\psi_g(Y;f', g') - \psi_g(Y;f,g))^2\Big] 
\in \bigO( \|f-f'\|^2 + \|g-g'\|^2 )
\;\;\textrm{ as }\;\;
\|f-f'\|^2, \|g-g'\|^2 \rightarrow 0.
\end{align*}
\label{asm:infFunRegularity}
\end{assumption}
\vspace{-0.2in}
Under the above assumptions, it is known \cite{emery98probstat,robins09quadraticvm}
that the \dss estimator on a single distribution achieves the mean squared error
(MSE) $\EE[(\Tfds-T(f))^2] \in
\bigO(n^{\minfours} + n^{-1})$ and further is asymptotically normal when $s>d/2$.
We have reviewed them along with a proof in Appendix~\ref{sec:appOneDistro}. 
Note that
\citet{robins09quadraticvm} analyse $\Tfds$ in the semiparametric setting.
We present a simpler self contained analysis that directly uses the VME
and has more interpretable assumptions. 
Bounding the bias and variance 
of the \dss estimator to establish the
convergence rate follows via a straightforward conditioning argument and Cauchy
Schwarz. However, an attractive property is that the analysis is agnostic to the
density estimator used provided it achieves the correct rates.

For the \loos estimator proposed in~\eqref{eqn:looEstimator}, 
we establish the following result.
\vspace{\parathmspacing}

\begin{theorem}[\textbf{Convergence of \loos Estimator for $T(f)$}]
Let $f\in\Sigma(s,L,B,B')$ and $\psi$ satisfy Assumption~\ref{asm:infFunRegularity}. 
Then,
$\EE[(\Tf - T(f))^2]$ is
$\bigO(n^{\frac{-4s}{2s+d}})$ when $s<d/2$ and $\bigO(n^{-1})$ when $s\geq d/2$.
\label{thm:convOneLoo}
\end{theorem}

The key technical challenge in analysing the \loos estimator (when compared to the
\dss estimator) is in bounding the variance with
several correlated terms in the summation. 
The bounded difference inequality is a popular trick used in such settings,
but this requires a supremum on the influence functions which leads to
significantly worse rates. 
Instead we use the Efron-Stein inequality which provides an integrated
version of bounded differences that can recover the correct rate when coupled with
Assumption~\ref{asm:infFunRegularity}. 
Our proof is contingent on the
use of the KDE as the density estimator.
While our empirical studies indicate that $\Tf$'s limiting distribution is normal
(Fig~\ref{fig:looHellingerAN}), the proof seems challenging due to the correlation
between terms in the summation.
We conjecture that $\Tf$ is indeed asymptotically normal but for now leave it as an
open problem.

We reiterate that while the convergence rates are the same for both \dss and \loos
estimators, the data splitting degrades empirical performance of $\Tfds$. 

Now we turn our attention to functionals of two distributions.
When analysing asymptotics we will assume that 
as $n,m \rightarrow \infty$,
$n/(n+m) \rightarrow \zeta \in (0,1)$.
%
Denote $N = n+m$.
For the \dss estimator~\eqref{eqn:dsEstimatorTwoD} we generalise 
\emph{our} analysis for one
distribution to establish the theorem below.
\\[\thmparaspacing]

\begin{theorem}[\textbf{Convergence/Asymptotic Normality of \dss Estimator
for $T(f,g)$}]
Let $f,g\in\Sigma(s,L,B,B')$ and $\psif,\psig$ satisfy 
Assumption~\ref{asm:infFunRegularity}. Then,
$\;\EE[(\Tfds - T(f,g))^2]$ is
$\bigO(n^{\frac{-4s}{2s+d}} + m^{\frac{-4s}{2s+d}})$ when $s<d/2$ 
and $\bigO(n^{-1} + m^{-1})$ when $s\geq d/2$.
Further, when $s>d/2$ and when $\psif, \psig \neq \zero$, $\Tfds$ is asymptotically
normal,
  \begin{align*}
  &\sqrt{N}(\Tfds - T(f,g)) \indistribution \numberthis 
  \label{eqn:asympNormalTwoDistro} 
  \Ncal \left( 0, \frac{1}{\zeta} \VV_{f} \left[\psi_f(X; f, g)\right]
   + \frac{1}{1 - \zeta} \VV_{g} \left[\psi_g(Y; f, g)\right] \right).
  \end{align*}
\label{thm:convTwoDS}
\end{theorem}
\vspace{-0.18in}

The asymptotic normality result is useful as it allows us to construct asymptotic
confidence intervals for a functional.
Even though the asymptotic variance of the influence function is 
not known, by Slutzky's theorem
any consistent estimate of the variance gives a valid asymptotic confidence
interval. In fact, we can use an influence function based estimator for the
asymptotic variance, since it is also a
differentiable functional of the densities.
We demonstrate this in our example in Appendix~\ref{sec:workedExample}.

The condition $\psif, \psig\neq \zero$ is somewhat technical. When \emph{both}
$\psi_f$ and $\psi_g$ are zero, the first
order terms vanishes and the estimator converges very fast (at rate $1/n^2$).
However, the asymptotic behavior of the estimator is unclear.
While this degeneracy occurs only on a meagre set, it does
arise for important choices. One example is the null hypothesis $f=g$ in two-sample
testing problems.

Finally, for the \loos estimator~\eqref{eqn:looEstimatorTwoD} on two distributions
we have the following result. \\[\thmparaspacing]

\begin{theorem}[\textbf{Convergence of \loos Estimator for $T(f,g)$}]
Let $f,g\in\Sigma(s,L,B,B')$ and $\psif,\psig$ satisfy
Assumption~\ref{asm:infFunRegularity}. Then,
$\;\EE[(\Tf - T(f,g))^2]$ is
$\bigO(n^{\frac{-4s}{2s+d}} + m^{\frac{-4s}{2s+d}})$ when $s<d/2$ and 
$\bigO(n^{-1} + m^{-1})$ when $s\geq d/2$.
\label{thm:convTwoLoo}
\end{theorem}

For many functionals, a H\"olderian assumption ($\Sigma(s,L)$) alone is 
sufficient to guarantee
the rates in Theorems~\ref{thm:convOneLoo},\ref{thm:convTwoDS} 
and~\ref{thm:convTwoLoo}. However, for
some functionals (such as the $\alpha$-divergences) we require
$\fhat, \ghat, f, g$ to be bounded above and below.
Existing results~\citep{krishnamurthy14renyi,birge95estimation} demonstrate 
that estimating such quantities is difficult without this assumption.

Now we turn our attention to the question of statistical difficulty. Via lower
bounds given by~\citet{birge95estimation} and~\citet{laurent1996efficient} 
we know that the \dss and \loos estimators
are minimax optimal when $s>d/2$ for functionals of one distribution. 
In the following theorem, we present a lower bound for estimating functionals of
two distributions. 
\\[\thmparaspacing]

\begin{theorem}[\textbf{Lower Bound for $T(f,g)$}]
Let $f,g \in \Sigma(s,L)$ and $\That$ be any estimator for $T(f,g)$.
Define $\tau = \min\{8s/(4s+d), 1\}$. Then there exists a strictly positive 
constant $c$ such
that,
\[
\liminf_{n\rightarrow \infty}\; \inf_{\That} \; \sup_{f,g \in \Sigma(s,L)} 
\EE\big[ (\That - T(f,g))^2 \big] \geq c \left( n^{-\tau} + m^{-\tau} \right).
\]
\label{thm:lowerbound}
\end{theorem}
\vspace{\thmparaspacing}
Our proof, given in Appendix~\ref{sec:appLowerBound}, is based on LeCam's 
method \cite{tsybakov08nonparametric} and
generalises the analysis of \citet{birge95estimation} for functionals 
of one distribution.
This establishes minimax
optimality of the \ds/\loos estimators for functionals of two distributions when
$s\geq d/2$. 
However, when $s<d/2$ there is a gap between our technique and the lower bound and
it is natural to ask if it is possible
to improve on our rates in this regime. 
A series of work \citep{birge95estimation, laurent1996efficient, 
kerkyacharian1996estimating} shows that, for integral 
functionals of one distribution, one can achieve the $n^{-1}$ rate when $s > d/4$ 
by estimating the second order term in the functional Taylor expansion. 
This second order correction was also done for polynomial functionals of two
distributions with similar statistical gains \citep{krishnamurthy14renyi}.
While we believe this is possible here, these estimators are conceptually
complicated and computationally expensive -- requiring
$O(n^3+m^3)$ effort when compared to the $O(n^2+m^2)$ effort for our estimator.
The first order estimator has a favorable
balance between statistical and computational efficiency.
Further, not much is known about the limiting distribution of
second order estimators.


\section{Comparison with Other Approaches}
\label{sec:comparison}
Estimation of statistical functionals under nonparametric assumptions
has received considerable attention over the last few decades.  A
large body of work has focused on estimating the Shannon entropy--
\citet{beirlant1997nonparametric} gives a nice review of results and
techniques.  More recent work in the single-distribution setting
includes estimation of R\'{e}nyi and Tsallis
entropies~\citep{leonenko2010statistical,pal2010renyi}.
There are also several papers extending some of these techniques to the
divergence estimation~\citep{krishnamurthy14renyi,poczos2011estimation,wang2009divergence,kallberg2012estimation,perez2008kullback}.

Many of the existing methods can be categorized as plug-in methods: they are based 
on estimating the densities either via a KDE or using $k$-Nearest Neighbors (\knn)
 and evaluating the functional on these estimates.
Plug-in methods are conceptually simple but unfortunately suffer several drawbacks.
First, they  typically have worse convergence rate than our approach, 
achieving the parametric rate only when $s \ge d$ as opposed to 
$s \ge d/2$~\citep{liu2012exponential,singh14plugin}.
Secondly, using either the KDE or \knn,
obtaining the best rates for plug-in methods requires undersmoothing the 
density estimate and we are not aware for principled approaches for hyperparameter 
tuning here.
In contrast, the bandwidth used in our estimators is the optimal bandwidth for 
density estimation, so a number of approaches such as cross validation are available.
This is convenient for a practitioner as our method \emph{does not require tuning hyper
parameters}.
Secondly, plugin methods based on the KDE always require computationally 
burdensome numeric
integration. In our approach, numeric integration can be avoided for many functionals 
of interest (See Table~\ref{tb:functionalDefns}).

There is also another line of work on estimating $f$-Divergences.
\citet{nguyen2010estimating} estimate $f$-divergences by solving a
convex program and analyse the technique when the likelihood ratio of the densities 
is in an RKHS. Comparing the theoretical results is not straightforward since
it is not clear how to port their assumptions to our setting.
Further, the size of the convex program increases with the sample size
which is problematic for large samples.  \citet{moon14fdivergence} use
a weighted ensemble estimator for $f$-divergences. They establish
asymptotically normality and the parametric convergence
rates only when $s \ge d$, which is a stronger smoothness assumption
than is required by our technique.  Both these works only consider
$f$-divergences.  Our method has wider applicability and includes
$f$-divergences as a special case.

\vspace{-0.05in}
\section{Experiments}
\label{sec:experiments}
\vspace{-0.05in}


We compare the estimators derived using our methods on a series of synthetic
examples in $1-4$ dimensions. We compare against the methods in
\cite{stowell2009fast,goria2005new,noughabi2013entropy,miller2003new,ramirez2009entropy,
perez2008kullback,poczos12divergence,poczos2011estimation}. Software for the
estimators is obtained either directly from the papers or from~\citet{szabo14ite}.
For the \ds/\loos estimators, we estimate the density
via a KDE with the smoothing kernels constructed using
Legendre polynomials \citep{tsybakov08nonparametric}. 
In both cases and for 
the plug in estimator we choose the bandwidth by
performing $5$-fold cross validation. The integration for the plug in estimator
is approximated numerically.

We test the estimators on a series of synthetic datasets in $1-4$ dimension.
The specifics of the densities used in the examples and methods compared to
are given  in Appendix~\ref{sec:appExperiments}.
The results are shown in Figures~\ref{fig:toyOne} and~\ref{fig:toyTwo}. 
We make the
following observations. In most cases the \loos estimator performs best. The \dss
estimator approaches the \loos estimator when there are many samples but is generally 
inferior to the \loos estimator with few samples. This, as we have explained before is
because data splitting does not make efficient use of the data.
The \knns estimator for divergences
\cite{poczos12divergence} requires choosing a $k$. For this estimator,
 we used the default setting for $k$ given in the software.
As performance is sensitive to the choice of $k$, it performs well in some cases but
poorly in other cases.
We reiterate that our estimators do not require any setting of hyperparameters.

Next, we present some results on asymptotic normality. We test the \dss and
\loos estimators on a $4$-dimensional Hellinger divergence estimation problem. We use
$4000$ samples for estimation. We repeat this experiment $200$ times and compare the
empiritical asymptotic distribution (i.e. the $\sqrt{4000}(\widehat{T} -
T(f,g))/\widehat{S}$ values where $\widehat{S}$ is the estimated asymptotic
variance) to a $\Ncal(0,1)$ 
distribution on a QQ plot. The results in Figure~\ref{fig:toyTwo} suggest that both
estimators are asymptotically  normal. 

\textbf{Image clustering: } We demonstrate  the use of our nonparametric divergence
estimators in an image clustering task on the ETH-80 datset.
Using our Hellinger divergence estimator we achieved an accuracy of $92.47\%$ 
whereas a naive spectral clustering approach achieved only $70.18\%$.
When we used a $k$-NN estimator for the Hellinger
divergence~\cite{poczos12divergence} we achieved $90.04\%$
which attests to the superiority of our method.
Since this is not the main focus of this work we defer this to
Appendix~\ref{sec:appExperiments}.

\insertFigToyTwo

\section{Conclusion}
\label{sec:conclusion}

We generalise 
existing results in Von Mises estimation by proposing an empirically superior 
\loos technique for estimating functionals and extending the framework
to functionals of two distributions.
We also prove a lower bound for the latter setting. 
We demonstrate the practical utility of our technique via
comparisons against other alternatives and an image clustering
application. An open problem arising out
of our work is to derive the limiting distribution of the \loos estimator.

{\small
\bibliography{./kky}
\bibliographystyle{unsrtnat}
}
\vfill

\pagebreak
\onecolumn
\setboolean{istwocolumn}{false}
\appendix
\section*{\Large Appendix}

\section{Auxiliary  Results}
\label{sec:appAncillaryResults}

\begin{lemma}[VME and Functional Taylor Expansion] 
\label{lem:vmeTaylorEqOneDistro}
Let $P, Q$ have densities $p, q$ and let $T(P) = \phi(\int \nu(p))$.
Then the first order VME of $T(Q)$ around $P$ reduces to a functional Taylor
expansion around $p$:
\begin{equation}
T(Q) = T(P) + T'(Q-P;P) + R_2 = T(p) +
   \phi'\left( \int \nu( p ) \right)\int \nu'(p) (q-p) + R_2
\end{equation}
\begin{proof}
It is 
sufficient to show that the first order terms are equal.
\begin{align*}
T'(Q-P;P) &= \partialfracat{t}{T( (1-t)P + tQ)}{0} 
  = \partialfrac{t}{} \phi \left(\int \nu( (1-t)p + tq) \right)\big|_{t=0}  \\
  &= \phi'\left( \int \nu( (1-t)p + tq) \right) \int \nu'((1-t)p + tq) (q-p)
      \big|_{t=0} \\
  &= \phi'\left( \int \nu(p) \right)\int \nu'(p) (q-p)
\end{align*}
\end{proof}
\end{lemma}

\begin{lemma}[VME and Functional Taylor Expansion - Two Distributions]
\label{lem:vmeTaylorEqTwoDistro}
Let $P_1, P_2, Q_1, Q_2$ be distributions with densities $p_1, p_2, q_1, q_2$.
Let $T(P_1, P_2) = \int \nu(p_1, p_2)$. Then,
\begin{align*}
T(Q_1, Q_2) &= T(P_1, P_2) + T'_1(Q_1-P_1; P_1, P_2) + T_2'(Q_2-P_2; P_1, P_2) +
              R_2  \numberthis \\
  &= T(P_1, P_2) + 
      \phi'\left(\int \nu(p) \right) \Big(
         \int \partialfrac{p_1(x)}{\nu(p_1(x), p_2(x))} (q_1 - p_1) \ud x + \\
  &\hspace{0.6in}
    \int \partialfrac{p_2(x)}{\nu(p_1(x), p_2(x))} (q_2 - p_2) \ud x \Big) + R_2
\end{align*}
\begin{proof}
Is similar to Lemma~\ref{lem:vmeTaylorEqOneDistro}.
\\[\thmparaspacing]
\end{proof}
\end{lemma}

\begin{lemma} Let $f, g$ be two densities bounded above and below  on
a compact space. Then for all $a, b$
\[
\| f^a - g^a \|_b \in O(\|f - g\|_b)
\]
\begin{proof}
Follows from the expansion,
\begin{equation*}
\int |f^a - g^a|^b = \int |g^a(x) + a(f(x)-g(x))g_*^{a-1}(x) - g^a(x)|^b \leq
a^b \sup |g_*^{b(a-1)}(x)| \int |f-g|^b.
\end{equation*}
Here $g_*(x)$ takes an
intermediate value between $f(x)$ and $g(x)$. In the second step we have used
the boundedness of $f$, $g$ to bound $f_*$.
\label{lem:densitypowers}
\end{proof}
\end{lemma}
Finally, we will make use of the Efron Stein inequality stated below in our analysis.
\\[\thmparaspacing]

\begin{lemma}[Efron-Stein Inequality]
Let $X_1, \dots, X_n, X'_1, \dots, X'_n$ be independent random variables where $X_i,
X'_i \in \Xcal_i$.
Let $Z = f(X_1, \dots, X_n)$ and $\Zii{i} = f(X_1, \dots, X'_i, \dots, X_n)$ where
$f:\Xcal_1\times\dots\times\Xcal_n \rightarrow \RR$. Then,
\[
\VV(Z) \;\leq\; \frac{1}{2} \;\EE\left[ \sum_{i=1}^n (Z - \Zii{i})^2 \right]
\]
\label{lem:efronstein}
\end{lemma}

\section{Review: \dss Estimator on a Single Distribution}
\label{sec:appOneDistro}

This section is intended to be a review of the data split estimator used in
\cite{robins09quadraticvm}.  The estimator was originally analysed in the
semiparametric setting. However, in order to be self contained we provide an
h analysis that directly uses the Von Mises Expansion.
We state our main result below. \\

\begin{theorem}
Suppose $f\in\Sigma(s,L,B,B')$ and $\psi$ satisfies
Assumption~\ref{asm:infFunRegularity}. Then,
$\;\EE[(\Tfds - T(f))^2]$ is
$\bigO(n^{\frac{-4s}{2s+d}})$ when $s<d/2$ and $\bigO(n^{-1})$ when $s>d/2$.
Further, when $s>d/2$ and when $\psif, \psig \neq \zero$, $\Tfds$ is asymptotically
normal.
  \begin{align*}
  &\sqrt{n}(\Tfds - T(f,g)) \indistribution \numberthis 
  \label{eqn:asympNormalTwoDistro} 
  \Ncal \left( 0, \frac{1}{\zeta} \VV_{f} \left[\psi_f(X; f, g)\right]
   + \frac{1}{1 - \zeta} \VV_{g} \left[\psi_g(Y; f, g)\right] \right)
  \end{align*}
\label{thm:convOneDS}
\end{theorem}

We begin the proof with a series of technical lemmas. \\

\begin{lemma}
The Influence Function has zero mean. i.e. $\EE_P[\psi(X;P)] = 0$.
\begin{proof}
$0 = T'(P-P;P) = \int \psi(x;P) \ud P(x)$.
\end{proof}
\label{lem:infFunMean}
\end{lemma}

Now we prove the following lemma on the
preliminary estimator $\Tfdsone$. \\

\begin{lemma}[Conditional Bias and Variance]
\label{thm:biasvar}
Let $\fhatone$ be a consistent estimator for $f$ in the $L_2$ metric. Let $T$ have
bounded second derivatives and let $\sup_x \psi(x; f)$ and $\VV_{X\sim
f}\psi(X;g)$ 
be bounded for all $g \in \Mcal$. Then,
the bias of the preliminary estimator $\Tfdsone$~\eqref{eqn:dsEstimator} conditioned on
$\Xone$ is $\bigO(\|f - \fhatone\|_2^2)$. The conditional 
variance is $\bigO(1/n)$.
\begin{proof}
First consider the conditional bias,
\begingroup
\allowdisplaybreaks
\begin{align*}
\EE_{\Xtwo} \left[ \Tfdsone - T(f) | \Xone \right] &=
  \EE_{\Xtwo} \left[ T(\fhatone) + \frac{2}{n} \nsumsechalf \psi(X_i; \fhatone)
      - T(f) | \Xone \right] \\
  &=  T(\fhatone) + \EE_f\left[ \psi(X; \fhatone) \right] - T(f) \in \bigO(\|\fhatone -
f\|_2^2). \numberthis
\end{align*}
\endgroup
The last step follows from the boundedness of the second derivative from which
the
first order functional Taylor expansion~\eqref{eqn:functaylor} holds.
The conditional variance is,
\begin{equation}
\VV_\Xtwo \left[ \Tfdsone|\Xone \right] = 
\VV_\Xtwo \left[ \frac{2}{n} \nsumsechalf \psi(X; \fhatone)\Big|\Xone \right] = 
 \frac{2}{n} \VV_f \left[ \psi(X; \fhatone) \right] \in \bigO(n^{-1}).
\end{equation}
\end{proof}
\end{lemma}

\begin{lemma}[Asymptotic Normality]
\label{thm:asympnormal}
Suppose in addition to the conditions in the lemma above we also have 
Assumption~\ref{asm:infFunRegularity} and 
$\|\fhatone - f \|_2 \in o_P(n^{-1/4})$ and
$\psi \neq \zero$.
Then, \
\[\sqrt{n}(\Tfds - T(f)) \indistribution \Ncal(0, \VV_f \psi(X; f)).
\]
\begin{proof}
We begin with the following expansion around $\fhatone$,
\begin{align}
T(f) &= T(\fhatone) + \int \psi(u; \fhatone) f(u) \ud \mu(u) + O(\|\fhatone - f\|^2).
\label{eqn:vmeTfhat}
\end{align}
First consider $\Tfdsone$. We can write
\begingroup
\allowdisplaybreaks
\begin{align*}
&\sqrt{\frac{n}{2}} \left( \Tfdsone - T(f) \right) = 
  \sqrt{\frac{n}{2}} \left( T(\fhatone) +
  \frac{2}{n} \nsumsechalf \psi(X_i; \fhatone) - T(f) \right) 
  \numberthis \label{eqn:cltdecomp} \\
&\hspace{0.3in}=   \sqrt{\frac{2}{n}}\nsumsechalf\left[
        \psi(X_i; \fhatone) - \psi(X_i; f) -
        \left(\int \psi(u;\fhatone) f(u) \ud u - \int \psi(u;f) f(u) \ud u\right)
        \right]  \\
&\hspace{0.5in} + \sqrt{\frac{2}{n}} \nsumsechalf \psi(X_i; f)  +\;\;\;
    \sqrt{n}\bigO\left( \|\fhatone - f\|^2 \right).
\end{align*}
\endgroup
In the second step we used the VME in~\eqref{eqn:vmeTfhat}. In the third step,
we added and subtracted $\sum_i \psi(X_i; f)$ and also added $\EE\psi(\cdot; f)
= 0$. Above, the third term is $o_P(1)$ as $\|\fhatone - f\|_2  \in o_P(n^{-1/4})$. 
The first term which we shall denote by $Q_n$ can also be
shown to be $o_P(1)$ via Chebyshev's inequality. It is sufficient to show
$\PP(|Q_n|>\epsilon | \Xone) \rightarrow 0$. First note that,
\begingroup
\allowdisplaybreaks
\begin{align*}
\VV [Q_n|\Xone] &= 
\VV \left[\sqrt{\frac{2}{n}}\nsumsechalf\left(
        \psi(X_i; \fhatone) - \psi(X_i; f) -
        \left(\int \psi(u;\fhatone) f(u) \ud u - \int \psi(u;f) f(u) \ud u\right)
        \right) \Big| \Xone \right] \\
&= \VV \left[ \psi(X; \fhatone) - \psi(X; f) -
        \left(\int \psi(u;\fhatone) f(u) \ud u - \int \psi(u;f) f(u) \ud u\right)
        \Big| \Xone \right] \\
&\leq \EE \left[ \left(\psi(X; \fhatone) - \psi(X; f)\right)^2 \right]
  \;\in\; \bigO(\|\fhatone - f\|^2) 
  \rightarrow 0, \numberthis
\end{align*}
\endgroup
where the last step follows from Assumption~\ref{asm:infFunRegularity}.
Now, $\PP(|Q_n|>\epsilon | X_1^n) \leq \VV(Q_n|\Xone) /\epsilon \rightarrow 0$.
Hence we have 
\[
\sqrt{\frac{n}{2}}(\Tfdsone - T(f)) = 
  \sqrt{\frac{2}{n}} \nsumsechalf \psi(X_i; f)  + o_P(1)
\]
We can similarly show 
\[
\sqrt{\frac{n}{2}}(\Tfdstwo - T(f)) = 
  \sqrt{\frac{2}{n}} \nsumsechalf \psi(X_i; f)  + o_P(1)
\]
Therefore, by the CLT and Slutzky's theorem,
\begin{align*}
\sqrt{n}(\Tfds - T(f)) &=  \frac{1}{\sqrt{2}}
  \left( \sqrt{\frac{n}{2}}(\Tfdsone - T(f)) + 
    \sqrt{\frac{n}{2}}(\Tfdstwo - T(f)) \right) \\
  &= n^{-1/2} \nsumwhole \psi(X_i; f)  + o_P(1) \;\;
  \indistribution \Ncal(0, \VV_f \psi(X;f)
\end{align*}
\end{proof}
\end{lemma}

We are now ready to prove Theorem~\ref{thm:convOneDS}. Note that the brunt of the
work for the \dss estimator was in analysing the preliminary estimator $\Tfds$.

\begin{proof}[Proof of Theorem~\ref{thm:convOneDS}]
We first note that in a \holder class, with $n$ samples the KDE achieves the rate
$\EE\|p - \phat\|^2 \in O(n^{\frac{-2s}{2s+d}})$. Then the bias of $\Tfds$ is,
\begin{align*}
\EE_{\Xone} \EE_{\Xtwo} \left[ \Tfdsone - T(f) | \Xone \right]
  &= \EE_{\Xone}\left[  O\left( \|f - \fhatone\|^2 \right) \right] 
   \in O\left(n^{\frac{-2s}{2s+d}} \right)
\end{align*} 
It immediately follows that $\EE\left[\Tfds-T(f)\right] \in 
O\left(n^{\frac{-2s}{2s+d}} \right)$.
For the variance, we use Theorem~\ref{thm:biasvar} and the Law of total
variance for $\Tfdsone$,
\begin{align*}
\VV_{\Xonetwo}\left[ \Tfdsone \right] 
  &= \frac{1}{n} \EE_{\Xone}\VV_f \left[ \psi(X; \fhatone, \ghat) \right] +
    +  \VV_{\Xone} \left[ \EE_{\Xtwo}\left[ \Tfds - T(f) | \Xone \right]
     \right] \\
  &\in O\left(\frac{1}{n} \right) +
    \EE_{\Xone} \left[O\left( \|f - \fhatone\|^4 \right) \right] \\
  &\in O\left( n^{-1} + n^{\frac{-4s}{2s+d}} \right)
\end{align*}
In the second step we used the fact that $\VV Z \leq \EE Z^2$. 
Further,
$\EE_{\Xone}\VV_f \left[ \psi(X; \fhatone) \right]$ is bounded since
$\psi$ is bounded.
The variance of $\Tfds$ can be bounded using the Cauchy Schwarz Inequality,
\begin{align*}
\VV\left[\Tfds\right] &= \VV\left[ \frac{\Tfdsone+ \Tfdstwo}{2} \right]
  = \frac{1}{4}\left( \VV\Tfdsone + \VV\Tfdstwo  + 2\Covar(\Tfdsone, \Tfdstwo) \right)\\
  &\leq \max\left( \VV\Tfdsone, \VV\Tfdstwo\right) 
  \in O\left( n^{-1} + n^{\frac{-4s}{2s+d}} \right)
\end{align*}
Finally for asymptotic normality, when $s > d/2$, $\EE\|\fhatone - f\|_2 
\in \bigO(n^{\frac{-s}{2s+d}})
\in o\,(n^{-1/4})$.
\end{proof}

\section{Analysis of \loos Estimator}
\label{sec:convOneLoo}

\begin{proof}[Proof of Theorem~\ref{thm:convOneLoo}]
First note that we can bound the mean squared error via the bias and variance terms.
\[
\EE [(\Tf-T(f))^2] \leq |\EE\Tf - T(f)|^2  + \EE[(\Tf - \EE\Tf)^2]
\]
The bias is bounded via a straightforward conditioning argument. 
\begin{align*}
\EE|\Tf-T(f)| &= \EE[ T(\fhatmi) + \psi(X_i;\fhatmi) - T(f) ] 
  = \EE_{\Xmi}\left[ \EE_{X_i}[ T(\fhatmi) + \psi(X_i;\fhatmi) - T(f)] \right] \\
  &= \EE_{\Xmi}\left[ \bigO(\|\fhatmi - f\|^2) \right]  \;
  \leq C_1n^{\frac{-2s}{2s+d}}
\numberthis\label{eqn:biasBound}
\end{align*}
for some constant $C_1$.
The last step follows by observing that the KDE achieves the rate 
$n^{\frac{-2s}{2s+d}}$ in integrated squared error.

To bound the variance we use the Efron-Stein inequality.
For this consider two sets of samples $\Xn=\{X_1, X_2, \dots, X_n\}$ and
$\Xn'=\{X'_1, X_2, \dots, X_n\}$ which are the same except for the 
first point. Denote the estimators obtained using $\Xn$ and $\Xn'$ by $\Tf$
and $\Tf'$ respectively. To apply Efron-Stein we shall bound $\EE[(\Tf - \Tf')^2]$.
Note that,
\begin{align*}
|\Tf - \Tf'| &\leq 
  \frac{1}{n}|\psi(X_1;\fhatmii{1}) - \psi(X'_1;\fhatmii{1})| +
  \frac{1}{n}\sum_{i\neq 1} |T(\fhatmi) - T(\fhatmi') | \\
  &\hspace{0.6in}
  + \frac{1}{n}\sum_{i\neq 1} |\psi(X_i;\fhatmi) -\psi(X_i;\fhatmi') | 
\numberthis \label{eqn:diffdecomp}
\end{align*}
The first term can be bounded by $2\|\psi\|_\infty/n$ using the boundedness of
$\psi$.
Each term inside the summation in the second term
in~\eqref{eqn:diffdecomp} can be bounded via,
\begin{align*}
|T(\fhatmi) - T(\fhatmi')|  &\leq
  L_\phi \int |\nu(\fhatmi) - \nu(\fhatmi')|  \leq
  L_\nu L_\nu \int |\fhatmi - \fhatmi'|  \numberthis \label{eqn:tempone} \\
  & \leq L_\phi L_\nu \int \frac{1}{nh^d} 
      \Big| K\left(\frac{X_1 - u}{h}\right) - 
      K\left(\frac{X'_1 - u}{h}\right) \Big| \ud u
  \leq \frac{\|K\|_\infty L_\phi L_\nu}{n}.
\end{align*}
The substitution $(X_1-u)/h = z$ for integration eliminates the $1/h^d$.
Here $L_\phi, L_\nu$ are the Lipschitz constants of $\phi, \nu$.
To apply Efron-Stein we need to bound the expectation of the LHS over 
$X_1, X'_1, X_2, \dots, X_n$. 
Since the first two terms in~\eqref{eqn:diffdecomp}
are bounded pointwise by $\bigO(1/n^2)$ they are also
bounded in expectation. 
By Jensen's inequality we can write,
\begin{align*}
|\Tf - \Tf'|^2 &\leq 
  \frac{12\|\psi\|^2_\infty}{n^2} + 
  \frac{ 3 \|K\|^2_\infty L^2_\phi L^2_\nu}{n^2}
  + \frac{3}{n^2}\left(\sum_{i\neq 1}|\psi(X_i;\fhatmi) -\psi(X_i;\fhatmi') |
      \right)^2
  \numberthis \label{eqn:esdecomp}
\end{align*}

For the third, such a pointwise bound does not hold so
we will directly bound the expectation.
\begin{align*}
  \sum_{1\neq i, j} 
  \EE \Big[|\psi(X_i;\fhatmi) -\psi(X_i;\fhatmi') | 
    |\psi(X_j;\fhatmj) -\psi(X_j;\fhatmj') | \Big] 
  \numberthis \label{eqn:thirdterm}
\end{align*}
We then have,
\begin{align*}
\EE \big[(\psi(X_i;\fhatmi) -\psi(X_i;\fhatmi') )^2\big]
&\leq \EE_{X_1,X'_1}\left[C \int |\fhatmi - \fhatmi'|^2 \right] \\
&\leq CB^2 \int \frac{1}{n^2h^{2d}}  
\left( K\left(\frac{x_1 - u}{h}\right) - 
      K\left(\frac{x'_1 - u}{h}\right) \right)^2 \ud x_1 \ud x'_1 u  \\
&\leq \frac{2 CB^2 \|K\|^2_\infty }{n^2}
\end{align*}
I the first step we have used Assumption~\ref{asm:infFunRegularity} and in the
last step  the substitutions $(x_1 - x_i)/h = u$ and $(x_1 - x_j)/h = v$
removes the $1/h^d$ twice.
Then, by applying Cauchy Schwarz each term inside the summation
in~\eqref{eqn:thirdterm} is $\bigO(1/n^2)$.

Since each term inside equation~\eqref{eqn:thirdterm} is $\bigO(1/n^2)$
and there are $(n-1)^2$ terms it is $\bigO(1)$.
Combining all these results with equation~\eqref{eqn:esdecomp} we get,
\[
\EE[(\Tf-\Tf')^2] \in \bigO\left(\frac{1}{n^2}  \right)
\]
Now, by applying the Efron-Stein inequality we get $\VV(\Tf) \leq \frac{C}{2n}$.
Therefore the mean squared error $\EE[(T - \Tf)^2] \in \bigO(n^{-\frac{4s}{2s+d}}
+ n^{-1})$ which completes the proof.
\end{proof}

\section{Proofs of Results on Functionals of Two Distributions}
\label{sec:appMultipleDistros}

\subsection{\dss Estimator}
\label{sec:dssEstimator}

We generalise the results in Appendix~\ref{sec:appOneDistro} to analyse the \dss
estimator for two distributions. As before we begin with a series of lemmas. \\

\begin{lemma}
The influence functions have zero mean. I.e.
\begin{align}
\EE_{P_1}[\psi_1(X;P_1; P_2)] = 0  \hspace{0.1in} \forall P_2 \in \Mcal
\hspace{0.5in} 
\EE_{P_2}[\psi_2(Y;P_1; P_2)] = 0  \hspace{0.1in} \forall P_1 \in \Mcal 
\end{align}
\begin{proof}
$0 = T'_i(P_i-P_i;P_1; P_2) = \int \psi_i(u;P_1, P_2) \ud P_i(u)$ for $i = 1,2$.
\end{proof}
\label{thm:infFunDeriv}
\end{lemma}

\begin{lemma} [Bias \& Variance of \eqref{eqn:dsEstimatorTwoD}]
Let $\fhatone, \ghatone$ be consistent estimators for $f, g$ in the $L_2$ sense. Let
$T$ have bounded second derivatives and let
$\sup_x \psi_f(x; f, g)$, $\sup_x \psi_g(x; f, g)$, $\VV_{f}\psi(X;f',g')$,
$\VV_{g}\psi_g(X;f',g')$ be bounded for all $f, f', g, g' \in \Mcal$. 
Then the bias of $\Tfdsone$ conditioned on $\Xone, \Yone$
is $|T - \EE[\Tfdsone|\Xone, \Yone] \in \bigO( \|f - \fhatone\|^2 + \|g - \ghatone\|^2)$.
The conditional variance is $\VV[\Tfdsone|\Xone, \Yone] \in \bigO(n^{-1} + m^{-1})$.
\label{thm:biasvarTwoDistro}
\begin{proof}
First consider the bias conditioned on $\Xone, \Yone$,
\begin{align*}
&\EE \left[ \Tfdsone - T(f, g) | \Xone, \Yone \right] \\
&\hspace{0.2in}= \EE\left[ T(\fhatone, \ghatone)
    + \frac{2}{n}\nsumsechalf \psi_f(X_i; \fhatone, \ghatone)
    + \frac{2}{m}\msumsechalf \psi_g(Y_j; \fhatone, \ghatone)
      - T(f, g) \Bigg| \Xone, \Yone \right] \\
&\hspace{0.2in}= T(\fhatone, \ghatone) + \int \psi_f(x; \fhatone, \ghatone) f(x) \ud \mu(x) +
   \int \psi_g(x; \fhatone, \ghatone) g(x) \ud \mu(x) - T(f, g)  \\
&\hspace{0.2in}= \bigO\left( \|f - \fhatone\|^2 + \|g - \ghatone\|^2 \right)
\end{align*}
The last step follows from the boundedness of the second derivatives from which
the first order functional Taylor expansion~\eqref{eqn:functaylortwoD} holds.
 The conditional variance is,
\begin{align*}
\VV\left[ \Tfdsone| \Xone, \Yone \right]
  &= \VV\left[ \frac{1}{n}\sum_{i=n+1}^{2n} \psi_f(X_i; \fhatone, \ghatone) \Big| \Xone \right]
   + \VV\left[ \frac{1}{m}\sum_{j=m+1}^{2m} \psi_g(Y_j; \fhatone, \ghatone) \Big| \Yone \right]
  \\
  &= \frac{1}{n} \VV_f \left[ \psi_f(X; \fhatone, \ghatone) \right] +
     \frac{1}{m} \VV_g \left[ \psi_g(Y; \fhatone, \ghatone) \right]
      \in \bigO\left(\frac{1}{n} + \frac{1}{m}\right)
\end{align*}
The last step follows from the boundedness of the variance of the influence
functions. 
\end{proof}
\end{lemma}

The following lemma characterises conditions for asymptotic normality. \\

\begin{lemma} [Asymptotic Normality] Suppose, in addition to the conditions in 
Theorem~\ref{thm:biasvarTwoDistro} above and the regularity
assumption~\ref{asm:infFunRegularity} we also have 
$\|\fhat - f\|\in o_P(n^{-1/4}), \|\ghat - g\| \in o_P(m^{-1/4})$ and 
$\psif, \psig \neq \zero$.
Then we have asymptotic Normality for $\Tfds$,
\ifthenelse{\boolean{istwocolumn}}
  {
  \begin{align*}
  &\sqrt{2N}(\Tfds - T(f,g)) \indistribution \numberthis 
  \label{eqn:asympNormalTwoDistro} \\
  &\hspace{0.05in}
  \Ncal \left( 0, \frac{1}{\zeta} \VV_{f} \left[\psi_f(X; f, g)\right]
   + \frac{1}{1 - \zeta} \VV_{g} \left[\psi_g(Y; f, g)\right] \right)
  \end{align*}
  }
  {
  \begin{align*}
  &\sqrt{N}(\Tfds - T(f,g)) \indistribution \numberthis 
  \label{eqn:asympNormalTwoDistro} 
  \Ncal \left( 0, \frac{1}{\zeta} \VV_{f} \left[\psi_f(X; f, g)\right]
   + \frac{1}{1 - \zeta} \VV_{g} \left[\psi_g(Y; f, g)\right] \right)
  \end{align*}
  }
\label{thm:asympNormalTwoDistro}
\begin{proof}
We begin with the following expansions around
$(\fhatone, \ghatone)$,
\begin{align*}
&T(f,g) = T(\fhatone, \ghatone) + \int \psi_f(u; \fhatone, \ghatone) f(u) \ud u +
  \int \psi_g(u; \fhatone, \ghatone) g(u) \ud u \;+  \\
    &\hspace{1in} \bigO\left( \|f - \fhatone\|^2 + \|g - \ghatone\|^2 \right) 
\end{align*}
Consider $\Tfdsone$. We can write
\begingroup
\allowdisplaybreaks
\begin{align*}
&\sqrt{\frac{N}{2}}(\Tfdsone - T(f)) 
      \numberthis \label{eqn:cltExpansion} \\
&\hspace{0.3in}= 
  \sqrt{\frac{N}{2}}\left( T(\fhatone, \ghatone) +
    \frac{2}{n}\nsumsechalf \psi_f(X_i;f, g) +
    \frac{2}{m}\msumsechalf \psi_g(Y_j;f, g) - T(f, g) \right)  \\
&\hspace{0.3in}
  = \sqrt{\frac{N}{2}}\Bigg( \frac{2}{n} \nsumsechalf \psi(X_i; \fhatone, \ghatone) +
     \frac{2}{m} \msumsechalf \psi(X_j; \fhatone, \ghatone)  
      - \EE_f\left[\psi(X;\fhatone,\ghatone)\right] \\ 
  &\hspace{0.6in}
      - \EE_g\left[\psi(X;\fhatone,\ghatone)\right]
      \Bigg) + \sqrt{N} O\left( \|f - \fhatone\|^2 + \|g - \ghatone\|^2
\right) \\
  &\hspace{0.3in}=\sqrt{\frac{2N}{n}} n^{-1/2} \nsumsechalf \left(
      \psi_f(X_i; \fhatone, \ghatone) - \psi_f(X_i; f, g) -
(\EE_f\psi_f(X;\fhatone,\ghatone) +
     \EE_f\psi_f(X;f,g) ) \right) + \\
      &\hspace{0.4in}
     \sqrt{\frac{2N}{m}} m^{-1/2} \msumsechalf \left(
      \psi_g(Y_j; \fhatone, \ghatone) - \psi_g(Y_j; f, g) -
(\EE_g\psi_g(Y;\fhatone,\ghatone) +
     \EE_g\psi_g(Y;f,g) ) \right) + \\ &\hspace{0.4in}
    \sqrt{\frac{2N}{n}} n^{-1/2} \nsumsechalf \psi_f(X_i; f, g)\,+\,
     \sqrt{\frac{2N}{m}} m^{-1/2} \msumsechalf \psi_g(Y_j; f, g)\,+\, \\
  &\hspace{0.6in}
      \sqrt{N} \bigO\left( \|f - \fhatone\|^2 + \|g - \ghatone\|^2 \right)
\end{align*}
\endgroup
The fifth term is $o_P(1)$ by the assumptions. The first and second terms are
also $o_P(1)$ . To see this, denote the first term by $Q_n$.
\begin{align*}
\VV\left[Q_n|\Xone, \Yone \right] &= \frac{N}{n}\VV_f \left[ \nsumsechalf
\left(
     \psi_f(X; \fhatone, \ghatone) - \psi_f(X; f, g) - (\EE_f\psi_f(X;\fhatone,\ghatone) +
     \EE_f\psi_f(X;f,g) ) \right) \right] \\
  &\leq \frac{N}{n} \EE_f\left[ \left(\psi_f(X_i; \fhatone, \ghatone) -
  \psi_f(X_i; f, g)\right)^2 \right] \rightarrow 0
\end{align*}
where we have used the regularity assumption~\ref{asm:infFunRegularity}.
Further,
$\PP(|Q_n| > \epsilon | \Xone, \Yone) \leq \VV[Q_n|\Xone, \Yone]
\ \epsilon \rightarrow 0$, hence the first term is $o_P(1)$. The proof for the
second term is similar.
Therefore we have,
\begin{align*}
\sqrt{\frac{N}{2}}(\Tfdsone - T(f)) = 
    \sqrt{\frac{2N}{n}} n^{-1/2} \nsumsechalf \psi_f(X_i; f, g)\,+\,
     \sqrt{\frac{2N}{m}} m^{-1/2} \msumsechalf \psi_g(Y_j; f, g)\,+\,
      o_P(1) 
\end{align*}
Using a similar argument on $\Tfdstwo$ we get,
\begin{align*}
\sqrt{\frac{N}{2}}(\Tfdstwo - T(f)) = 
    \sqrt{\frac{2N}{n}} n^{-1/2} \sum_{i=1}^{n/2} \psi_f(X_i; f, g)\,+\,
     \sqrt{\frac{2N}{m}} m^{-1/2} \sum_{j=1}^{m/2} \psi_g(Y_j; f, g)\,+\,
      o_P(1) 
\end{align*}
Therefore,
\begin{align*}
\sqrt{N}(\Tfdstwo - T(f)) &= 
    \sqrt{2} \left(
    \sqrt{\frac{2N}{n}} n^{-1/2} \sum_{i=1}^{n} \psi_f(X_i; f, g)\,+\,
     \sqrt{\frac{2N}{m}} m^{-1/2} \sum_{j=1}^{m} \psi_g(Y_j; f, g)\,
    \right)\, +\,
      o_P(1)  \\
    &=
    \sqrt{\frac{N}{n}} n^{-1/2} \sum_{i=1}^{2n} \psi_f(X_i; f, g)\,+\,
     \sqrt{\frac{N}{m}} m^{-1/2} \sum_{j=1}^{2m} \psi_g(Y_j; f, g)\,+\,
      o_P(1) 
\end{align*}
By the CLT and
Slutzky's theorem this converges weakly to
the RHS of~\eqref{eqn:asympNormalTwoDistro}.
\end{proof}
\end{lemma}

We are now ready to prove the rates of convergence for the \dss estimator
in the \holder class. \\

\begin{proof}[Proof of Theorem~\ref{thm:convOneDS}].
We first note that in a \holder class, with $n$ samples the KDE achieves the rate
$\EE\|p - \phat\|^2 \in O(n^{\frac{-2s}{2s+d}})$. Then the bias for the
preliminary estimator $\Tfdsone$ is,
\begin{align*}
\EE \left[ \Tfdsone - T(f, g) | \Xone, \Yone \right]
  &= \EE_{\Xone, \Yone}\left[  O\left( \|f - \fhatone\|^2 + \|g - \ghatone\|^2 \right)
          \right] \\
  &\in O\left(n^{\frac{-2s}{2s+d}} + m^{\frac{-2s}{2s+d}} \right) 
\end{align*} 
The same could be said about $\Tfdstwo$.
It therefore follows that
\[
\EE\left[ \Tfds-T\right] = \EE\left[ \frac{1}{2}\left(\Tfdsone-T(f)\right) + 
  \frac{1}{2}\left(\Tfdstwo-T(f)\right) \right]
  \in O\left(n^{\frac{-2s}{2s+d}} + m^{\frac{-2s}{2s+d}} \right) 
\]
For the variance, we use Theorem~\ref{thm:biasvarTwoDistro} and the Law of total
variance to first control $\VV\Tfdsone$,
\begin{align*}
\VV\left[ \Tfdsone \right] 
  &= \frac{1}{n} \EE\left[\VV_f \left[ \psi_f(X; \fhatone, \ghatone)|\Xone \right]\right]+
     \frac{1}{m} \EE\left[\VV_g \left[ \psi_g(Y; \fhatone, \ghatone)|\Yone \right]\right] 
     \\  &\hspace{0.3in}  +  \VV\
     \left[ \EE\left[ \Tf - T(f,g) | \Xone \Yone \right]
     \right] \\
  &\in O\left(\frac{1}{n} + \frac{1}{m}\right) +
    \EE \left[O\left( \|f - \fhatone\|^4 + \|g - \ghatone\|^4 \right)
                              \right] \\
  &\in O\left(n^{-1} + m^{-1}+ n^{\frac{-4s}{2s+d}} + m^{\minfours} \right)
\end{align*}
In the second step we used the fact that $\VV Z \leq \EE Z^2$. 
Further,
$\EE_{\Xone}\VV_f \left[ \psi_f(X; \fhatone, \ghatone) \right]$,
$\EE_{\Yone}\VV_g \left[ \psi_g(Y; \fhatone, \ghatone) \right]$ are bounded since
$\psi_f$, $\psi_g$ are bounded. Then by applying the Cauchy Schwarz inequality
as before we get $\VV\Tfds \in 
  O\left(n^{-1} + m^{-1}+ n^{\frac{-4s}{2s+d}} + m^{\minfours} \right)$.

Finally when $s > d/2$, we have the required $o_P(n^{-1/4}), o_P(m^{-1/4})$
rates on $\|\fhat - f\|$ and $\|\ghat - g\|$
which gives us asymptotic normality.
\end{proof}

\subsection{\loos Estimator}
\label{sec:looEstimatorTwoD}

\begin{proof}[Proof of Theorem~\ref{thm:convTwoLoo}]
Assume w.l.o.g that $n>m$.
As before, the bias follows via conditioning.
\begin{align*}
\EE|\Tf - T(f,g)| &= \EE[ T(\fhatmi, \ghatmi) + \psif(X_i;\fhatmi, \ghatmi) 
  + \psig(Y_i; \fhatmi, \ghatmi) - T(f,g) ] \\
  &= \EE\left[ \bigO(\|\fhatmi - f\|^2 + \|\ghat - g\|^2)\right]  
  \leq C_1 (n^{\mintwos} + m^{\mintwos})  
\end{align*}
for some constant $C_1$.

To bound the variance we use the Efron-Stein inequality. 
Consider the samples  $\{X_1, \dots, X_n, Y_1, \dots, Y_m\}$ and
 $\{X'_1, \dots, X_n, Y_1, \dots, Y_m\}$ and denote the estimates obtained by $\Tf$
and $\Tf'$ respectively.
Recall that we need to bound $\EE[(\Tf - \Tf)^2]$.
Note that,
\begin{align*}
& |\Tf - \Tf'| \leq 
  \frac{1}{n}|\psif(X_1;\fhatmii{1}, \ghatmii{1}) - \psif(X'_1;\fhatmii{1},\ghatmii{1})| +
\\
  &\hspace{0.2in} 
  \frac{1}{n}\sum_{i\neq 1} |T(\fhatmi, \ghatmi) - T(\fhatmi',\ghatmi) | +
  |\psif(X_i;\fhatmi, \ghatmi) -\psif(X_i;\fhatmi', \ghatmi) |   +
  |\psig(Y_i;\fhatmi, \ghatmi) -\psig(Y_i;\fhatmi', \ghatmi) |  
\end{align*}
The first term can be bounded by $2\|\psif\|_\infty/n$ using the boundedness of
the influence function on bounded densities. 
By using an argument similar to Equation~\eqref{eqn:tempone} in the one
distribution case, we can also bound each term inside the summation of the
second term via,
\[
|T(\fhatmi,\ghatmi) - T(\fhatmi', \ghatmi)| \leq
\frac{ \|K\|_\infty L_\phi L_\nu }{n}
\]
Then, by Jensen's inequality we have,
\begin{align*}
|\Tf - \Tf'|^2 &\leq \frac{8\|\psif\|^2_\infty}{n^2} +
  \frac{4\|K\|^2_\infty L^2_\phi L^2_\nu}{n^2} 
  +\frac{4}{n^2} \left(\sum_{i\neq 1} 
      |\psif(X_i;\fhatmi, \ghatmi) -\psif(X_i;\fhatmi', \ghatmi) |\right)^2 \\
 &\hspace{0.4in}  +\frac{4}{n^2}  \left(\sum_{i\neq 1} 
      |\psig(Y_i;\fhatmi, \ghatmi) -\psig(Y_i;\fhatmi', \ghatmi) |\right)^2
\end{align*}
The third and fourth terms can be bound in expectation using a similar technique
to bound the third term in equation~\ref{eqn:tempone}. 
Precisely, by using Assumption~\eqref{asm:infFunRegularity} and Cauchy Schwarz
we get,
\begin{align*}
\EE\big[|\psif(X_i;\fhatmi,\ghatmi) - \psif(X_i;\fhatmi', \ghatmi)|
|\psif(X_j;\fhatmj,\ghatmj) - \psif(X_j;\fhatmj', \ghatmj)|\big]
  &\leq \frac{2CB^2\|K\|^2_\infty}{n^2} \\
\EE\big[|\psig(Y_i;\fhatmi,\ghatmi) - \psig(Y_i;\fhatmi', \ghatmi)|
|\psig(Y_j;\fhatmj,\ghatmj) - \psig(Y_j;\fhatmj', \ghatmj)|\big]
  &\leq \frac{2CB^2\|K\|^2_\infty}{n^2}
%
\end{align*}
This leads us to a $\bigO(1/n^2)$ bound for $\EE[(\Tf - \Tf')^2]$,
\begin{align*}
\EE[(\Tf-\Tf')^2] \leq \frac{ 8\|\psif\|^2_\infty 
  + 4\|K\|^2_\infty L^2_\phi L^2_\nu + 
  16CB^2\|K\|^2_\infty}{n^2}
\end{align*}

Now consider, the set of samples
$\{X_1, \dots, X_n, Y_1, \dots, Y_m\}$ and
 $\{X_1, \dots, X_n, Y'_1, \dots, Y_m\}$ and denote the estimates obtained by $\Tf$
and $\Tf'$ respectively. Note that some of the $Y$ instances are repeated but
each point occurs at most $n/m$ times.
The remaining argument is exactly the same except that we need to account for
this repetition. We have,
\begin{align*}
&|\Tf - \Tf'| \leq 
  \frac{n}{m} \frac{1}{n}|\psif(X_1;\fhatmii{1}, \ghatmii{1}) - 
  \psif(X'_1;\fhatmii{1},\ghatmii{1})| \,+\,
  \frac{n}{m} \frac{1}{n}\sum_{i\neq 1} \Big(|T(\fhatmi, \ghatmi) 
      - T(\fhatmi',\ghatmi) | + \\
&\hspace{0.4in} 
  |\psif(X_i;\fhatmi, \ghatmi) -\psif(X_i;\fhatmi', \ghatmi) |   \,+
  |\psig(Y_i;\fhatmi, \ghatmi) -\psig(Y_i;\fhatmi', \ghatmi) |   \Big) 
\numberthis \label{eqn:temptwo}
\end{align*}
And hence,
\begin{align*}
\EE[(\Tf - \Tf')^2] &\leq 
  \frac{\|\psig\|^2_\infty}{m^2} + 
  \frac{n^2}{m^4} 4\|K\|^2_\infty L^2_\phi L^2_\nu 
  + \bigO\left(\frac{n^4}{m^6}\right)
\end{align*}
where the last two terms of~\eqref{eqn:temptwo} are bounded by $\bigO(n^4/m^6)$
after squaring and then taking the expectation.
We have been a bit sloppy by bounding the difference by $n/m$ and not
$\ceil{n/m}$ but it is clear that this doesn't affect the rate.

Finally by the Efron Stein inequality we have 
\[
\VV(\Tf) \in \bigO\left(\frac{1}{n} + \frac{n^4}{m^5}\right)
\]
which is $\bigO(1/n + 1/m)$ if $n$ and $m$ are of the same order. This is the
case if for instance there exists $\zeta_l, \zeta_u \in (0,1)$ such that
$\zeta_l \leq n/m \leq \zeta_u$.

Therefore the mean squared error is 
$\EE[(T - \Tf)^2] \in \bigO(n^{-\frac{4s}{2s+d}} + m^{-\frac{4s}{2s+d}}
+ n^{-1} + m^{-1})$ which completes the proof.
\end{proof}

\section{Proof of Lower Bound (Theorem~\ref{thm:lowerbound})}
\label{sec:appLowerBound}

We will prove the lower bound in the bounded \holder class $\Sigma(s,L,B,B')$
noting that the lower bound also applies to $\Sigma(s,L)$.
Our main tool will be LeCam's method where we reduce the estimation problem to a
testing problem. In the testing problem we construct a set of alternatives satisfying
certain separation properties from the null. For this we will use some
technical results from \citet{birge95estimation} and
\cite{krishnamurthy14renyi}.
First we state LeCam's method below adapted to our setting.
We define the squared Hellinger Divergence between two distributions
$P,Q$ with densities $p, q$ to be
\[
H^2(P,Q) = \int \big(\sqrt{p(x)} - \sqrt{q(x)}\big)^2 \ud x
  = 2 - 2\int p(x)q(x) \ud x
\]
\\[\thmparaspacing]

\begin{theorem}
\label{thm:lecam}
Let $T:\Mcal \times \Mcal \rightarrow \RR$. Consider a parameter space $\Theta
\subset \Mcal \times \Mcal$ such that $(f,g) \in \Theta$ and $(\plambda, \qlambda)
\in \Theta$ for all $\lambda$ in some index set $\Lambda$. 
Denote the distributions of $f,g,\plambda,\qlambda$ by $F,G,\Plambda,\Qlambda$
respectively. 
Define 
$\PQbar = \frac{1}{|\Lambda|} \sum_{\lambda \in \Lambda} \Plambda^n \times 
\Qlambda^m$.
If, there exists $(f,g) \in \Theta$,
 $\gamma < 2$ and $\beta > 0$ such that the following two conditions
are satisfied
\begin{align*}
& H^2(\FnGm, \PQbar) \leq  \gamma \\
& T(\plambda,\qlambda) \geq T(f,g) + 2\beta \;\;\; \forall\; \lambda \in \Lambda
\end{align*}
then,
\[
\inf_{\That} \sup_{(f,g) \in \Theta} 
\PP\left( |\That - T(f,g)| > \beta\right) \frac{1}{2}\left(1 -
\sqrt{\gamma(1-\gamma/4)} \right) > 0.
\]
\begin{proof}
The proof is a straightforward modification of Theorem 2.2 of
\citet{tsybakov08nonparametric} which we provide here 
for completeness.

Let $\Theta_0 = \{(p,q) \in \Theta | T(p,q) \leq T(f,g)\}$ and 
$\Theta_1 = \{(p,q) \in \Theta| T(p,q) \geq T(f,g) + 2\beta\}$. Hence
$(f,g) \in \Theta_0$ and $\pqlambda \in \Theta_1$ for all $\lambda \in \Lambda$.
Given $n$ samples from $p'$ and $m$ samples from $q'$
consider the simple vs simple hypothesis testing problem of 
$H_0:(p',q') \in \Theta_0$ vs $H_1: (p',q') \in \Theta_1$. The probability of error
$p_e$ of any test $\Psi$ test is lower bounded by 
\[
p_e \geq \frac{1}{2}\left(1 - \sqrt{H^2(\FnGm,\PQbar)(1-
H^2(\FnGm,\PQbar))/4}\right).
\]
See Lemma 2.1, Lemma 2.3 and Theorem 2.2 of \citet{tsybakov08nonparametric}.
Therefore,
\[
\inf_\psi \sup_{(p',q')\in \Theta_0, (p'',q'')\in \Theta_0}
p_e \geq \frac{1}{2}\left(1-\sqrt{\gamma(1-\gamma/4)}\right)
\]
If we make an error in the testing problem the error in estimation is least 
$\beta$ in the
estimation problem which completes the proof of the theorem.
\end{proof}
\end{theorem}
\vspace{\thmparaspacing}

Consider the set $\Gamma = \{-1,1\}^\ell$ and a set of densities
$\pgamma = f(1 + \sum_{j=1}^\ell \gamma_j v_j)$ indexed by each $\gamma \in \Gamma$.
Here $f$ is itself a density and the $v_j$'s are perturbations on $f$.
We will also use the following result from 
\citet{birge95estimation} which bounds the
Hellinger divergence between the product distribution $\Fn$ and the mixture
product distribution $\Pbarn = \frac{1}{|\Gamma|} \sum_{\gamma \in \Gamma}
\Pgamman$. \\[\thmparaspacing]

\begin{proposition}
\label{prop:hellinger}
Let $\{R_1, \dots, R_\ell\}$ be a partition of $[0,1]^d$.
Let $\rho_j$ is zero except on $R_j$ and satisfies
$\|\rho_j\|_\infty \leq 1$,  $\int \rho_j f = 0$ and $\int \rho_j^2 f =\alpha_j$.
Further, denote $\alpha = \sum_j\|\rho_j\|_\infty$,
$s = n\alpha^2 \sup_j P(R_j)$ and
$c = n \sup_j \alpha_j$. Then,
\[
H^2(\Fn, \Pbarn) \leq \frac{n^2}{3} \sum_{j=1}^\ell \alpha_j^2.
\]
\end{proposition}

We also use the following technical result from \citet{krishnamurthy14renyi} and adapt it
to our setting. \\[\thmparaspacing]

\begin{proposition}[Taken from \cite{krishnamurthy14renyi}]
\label{prop:construction}
Let $R_1, \dots, R_\ell$ be a partition of $[0,1]^d$  each having size $\ell^{-1/d}$.
There exists functions $u_1,\dots,u_\ell$ such that,
\begin{align*}
&\supp(u_j) \subset \{x| B(x,\epsilon) \subset R_j\}, \;\;\;\;\;\;\;\;
\int u_j^2 \in \Theta(\ell^{-1}), \;\;\;\;\;\;\;\;
\int u_j = 0, \\
&\int \psi_f(x; f,g) u_j(x)  = \int \psi_g(x; f,g) u_j(x) = 0, \;\;\;\;\;\;\;\;
\|D^r u_j\|_\infty \leq \ell^{r/d} \;\; \forall r  \textrm{ s.t } \sum_j r_j \leq s + 1
\end{align*}
where $B(x,\epsilon)$ denotes an $L_2$ ball around $x$ with radius $\epsilon$. Here
$\epsilon$ is any number between $0$ and $1$.
\begin{proof} 
For this we use an orthonormal system of $q \;(>4)$ functions
 on $(0,1)^d$ satisfying  $\phi_1 = 1$, $\supp(\phi_j) \subset [\epsilon,
1-\epsilon]^d$ for any $\epsilon > 0$ and $\|D^r \phi_j\|_\infty \leq J$
for some $J < \infty$.
Now for any given functions $\eta_1, \eta_2$
we can find a function $\upsilon$ such that $\upsilon \in \spann(\{\phi_j\})$,
$\int \upsilon \phi_1 = \int \upsilon \eta_1 = \int \upsilon \eta_2 = 0$.
Write $\upsilon = \sum_i c_j\phi_j$. Then $D^r\upsilon = \sum_j c_j D^r \phi_j$ which
implies $\|D^r\upsilon\|_\infty \leq K\sqrt{q}$. Let $\nu(\cdot) =
\frac{1}{J\sqrt{q}} \upsilon(\cdot)$. Clearly, $\int \nu^2$ is upper and lower
bounded and $\|D^r\nu\|_\infty \leq 1$.

To construct the functions $u_j$, we map $(0,1)^d$ to $R_j$ by appropriately scaling
it. Then, $u_j(x) = \nu(m^{1/d}(x - \jb))$ where $\jb$ is the point corresponding to
$\zero$ after mapping. 
 Moreover let $\eta_1$ be $\psif(\cdot; f,g)$ constrained to $R_j$ (and scaled back to fit
$(0,1)^d$). Let $\eta_2$ be the same with $\psig$. 
Now, $\int_{R_j} u_j^2 = \frac{1}{\ell} \int \nu^2 \in \Theta(\ell^{-1})$.
Also, clearly $\|D^r u_j\| \leq m^{r/d}$. All $5$ conditions above are satisfied.
\end{proof}
\end{proposition}

We now have all necessary ingredients to prove the lower bound.

\begin{proof}[Proof of Theorem~\ref{thm:lowerbound}]
To apply Theorem~\ref{thm:lecam} we will need to construct the set of alternatives
$\Lambda$ which contains tuples $\pqlambda$ that satisfy the conditions of
Theorem~\ref{thm:lecam}. 
First apply Proposition~\ref{prop:construction} with $\ell = \ell_1$ to obtain the
index set $\Gammat = \{-1,1\}^{\ell_1}$ and the functions $u_1,\dots,u_{\ell_1}$.
Apply it again with $\ell = \ell_2$ to obtain the
index set $\Delta = \{-1,1\}^{\ell_2}$ and the functions $v_1,\dots, v_{\ell_2}$.
Define $\Gamma,\Delta$ be the following set of functions which are perturbed around
$f$ and $g$ respectively,
\begin{align*}
\Gamma &= \Big\{\pgamma = f + K_1 \sum_{j=1}^{\ell_1} \gamma_j u_j | \gamma \in \Gammat
  \Big\}
\\
\Delta &= \Big\{\qdelta = g + K_2 \sum_{j=1}^{\ell_2} \delta_j v_j | \delta \in \Deltat
  \Big\}
\end{align*}
Since the perturbations in Proposition~\ref{prop:construction} 
are condensed into the small $R_j$'s it invariably violates
the \holder assumption. 
The scaling $K_1$ and $K_2$ are necessary to shrink the perturbation and ensure that
$\pgamma, \qdelta \in \Sigma(s,L)$.
By following essentially an identical argument to 
\cite{krishnamurthy14renyi} (Section E.2) we have that $\pgamma \in \Sigma(s,L)$ if
$K \asymp \ell_1^{-s/d}$ and $\qdelta \in \Sigma(s,L)$ if $K_2 \asymp \ell_2^{-s/d}$.
We will set $\ell_1$ and $\ell_2$ later on to obtain the required rates.
For future reference denote 
$\Pbarn = \frac{1}{|\Gamma|} \sum_{\gamma \in \Gamma} \Pgamma^n$ and
$\Qbarm = \left(\frac{1}{|\Delta|} \sum_{\delta \in\Delta} \Qdelta^m\right) $.

Now our set of  alternatives are formed by the product of $\Gamma$ and $\Delta$
\[
\Lambda = \Gamma \times \Delta = \left\{ (\pgamma,\qdelta) | \pgamma \in \Gamma, 
\qdelta \in \Delta \right\} 
\]

First note that for any $\pqlambda = (\pgamma,\qdelta) \in \Lambda$, 
by  the second order functional
Taylor expansion we have,
\begin{align*}
T\pqlambda &= T(f,g) + \int \psif(x;f,g) \plambda + \int \psig(x; f,g) \qlambda
  + R_2
\end{align*}
By Lemma~\ref{thm:infFunDeriv} and the construction
 the first order terms vanish since,
\[
\int \psif(x; f,g) \left(f + K_1 \sum_j \gamma_j u_j\right)
= K_1 \sum_j\gamma_j \int \psif(x;f,g) u_j = 0.
\]
The same is true for $\int \psig(x; f,g)$. The second order term can be upper bounded
by 
\begin{align*}
R_2 &= \phi''\left(\int \nu(f^*,g^*)\right) \bigg(
  \int \frac{\partial^2 \nu(f^*(x), g^*(x))}{\partial f^2(x)} (\plambda - f)^2 +
  \int \frac{\partial^2 \nu(f^*(x), g^*(x))}{\partial g^2(x)} (\qlambda - g)^2 + \\
  &\hspace{0.4in}
  2\int \frac{\partial^2 \nu(f^*(x), g^*(x))}{\partial g(x)\partial g(x)} 
    (\plambda - f)(\qlambda - g) \bigg) \\
  &\geq \sigma_{\min}\left( \|\plambda - f\|^2 + \|\qlambda - g\|^2 \right)
  \;\geq \sigma_{\min}\left(K_1^2 + K_2^2\right)
\end{align*}
For the second step note that $(f^*,g^*)$ lies in line segment between $(\plambda,
\qlambda)$ and $(f,g)$ and is therefore both upper and lower bounded. Therefore, 
the Hessian evaluated at $(f^*, g^*)$ is strictly positive definite
with some minimum eigenvalue $\sigma_{min}$.
For the third step we have used that $(\plambda - f, \qlambda - g) = 
  (K_1\sum_{j=1}^{\ell_1} \gamma_j u_j, K_2\sum_{j=1}^{\ell_2} \delta_j v_j)$
and that the $u_j$'s are orthonormal and $\|u_j\|_2 = 1$.
This establishes the $2\beta$ separation between the null and the alternative as
required by Theorem~\ref{thm:lecam} with 
$\beta = \sigmamin(K_1^2 + K_2^2)/2$. Precisely,
\[
T\pqlambda \geq T(f,g) + \bigO(\ell_1^{-2s/d} + \ell_2^{-2s/d}) 
\]

Now we need to bound the Hellinger separation, between $\FnGm$ and $\PQbar$.
First note that by our construction,
\[
\PQbar = \frac{1}{|\Lambda|} \sum_{\lambda \in \Lambda} \Plambda^n \times \Qlambda^m
 = \left(\frac{1}{|\Gamma|} \sum_{\gamma \in \Gamma} \Pgamma^n\right) \times
 \left(\frac{1}{|\Delta|} \sum_{\delta \in\Delta} \Qdelta^m\right) 
  = \Pbarn \times \Qbarm
\]
By the tensorization property of the Hellinger affinity we have,
\begin{align*}
H^2(\FnGm, \PQbar)  
  &= 2\left( 1 - \left( 1- \frac{H^2(\Fn,\Pbarn)}{2}\right)
      \left( 1- \frac{H^2(\Gm,\Qbarm)}{2}\right) \right)  \\
  &\leq H^2(\Fn,\Pbarn) + H^2(\Gm,\Qbarm)
\end{align*}
We now apply Proposition~\ref{prop:hellinger} to bound each Hellinger divergence.
If we  denote $\rho_j(\cdot) = K_1 u_j(\cdot)/f(\cdot)$ then we see that the
$\rho_j$'s satisfy the conditions of the proposition and further $\pgamma = f(1 +
\sum_j \gamma_j \rho_j)$ allowing us to use the bound. 
Accordingly $\alpha_j = \int \rho_j^2 f \leq C K_1^2 /\ell_1 $ for some $C$.
Hence, 
\[
H^2(\Fn,\Pbarn) \leq \frac{n^2}{3} \sum_{j=1}^m \alpha_j^2
 \leq \frac{C n^2 K_1^4}{\ell_1} \in \bigO( n^2 \ell_1^{-\frac{4s+d}{d}}). 
\]
A similar argument yields
$H^2(\Gm,\Qbarm) \in \bigO(m^2 \ell_2^{-\frac{4s+d}{d}})$. 
If we pick $\ell_1 = n^{\frac{2d}{4s+d}}$ and $\ell_2 = m^{\frac{2d}{4s+d}}$
and hence $K_1 = n^{\mintwos}$ and $K_2 = m^{-\mintwos}$, then we have
that the Hellinger separation is bounded by a constant.
\[
H^2(\FnGm,\PQbar)
  \leq H^2(\Fn,\Pbarn) + H^2(\Gm,\Qbarm) \in \bigO(1)
\]
Further, the error is larger than $\beta \asymp K_1^s + K_2^2 
\asymp n^{\minfours} + m^{\minfours}$.

The first part of the lower bound for $\tau = 8s/(4s+d)$ 
is concluded by Markov's inequality,
\[
\frac{\EE\big[(\That - T(f,g))^2]}{ (n^{-\tau/2} + m^{-\tau/2})^2 }
\leq \PP\left(|\That - T(f,g)| > (n^{-\tau/2} + m^{-\tau/2}) \right) > c
\]
where we note that $(n^{-\tau/2} + m^{-\tau/2})^2 \asymp n^{-\tau} + m^{-\tau}$.
The $n^{-1} + m^{-1}$ lower bound is straightforward as as we cannot do better than the
the parametric rate \cite{bickel1988estimating}. 
See \cite{krishnamurthy14renyi} for an proof that uses a contradiction argument in
the setting $n=m$. 
\end{proof}

\section{An Illustrative Example - The Conditional Tsallis Divergence}
\label{sec:workedExample}

In this section we present a step by step guide on applying our framework to
estimating any desired functional.
We choose the Conditional Tsallis divergence because pedagogically it is a
good example in Table~\ref{tb:functionalDefns} to illustrate the technique.
By following a similar procedure, one may derive an estimator for any desired
functional.
The estimators are derived in Section~\ref{sec:tsallisEstimator} and
in Section~\ref{sec:tsallisAnalysis} we discuss conditions for the theoretical
guarantees and asymptotic normality.

The Conditional \tsallis divergence
($\alpha \neq 0, 1$) between $X$ and $Y$ conditioned on $Z$  can be written
in terms of joint densities $\pxz, \pyz$. 
\ifthenelse{\boolean{istwocolumn}}
  {
  \begin{align*}
  &\tsalliscd(p_{X|Z} \| p_{Y|Z}; p_Z) = \tsalliscd(\pxz, \pyz) =  \\
   &=\int  p_Z(z)\frac{1}{\alpha-1}\left(\int p_{X|Z}^{\alpha}(u,z)
      p_{Y|Z}^{1-\alpha}(u,z) \ud u -1\right) \ud z  \\
    &= \frac{1}{1- \alpha} + \frac{1}{\alpha-1}
    \int p^{\alpha}_{XZ}(u,z) p^{\beta}_{YZ}(u,z) \ud u \ud z
  \end{align*}
  }
  {
  \begin{align*}
  \tsalliscd(p_{X|Z} \| p_{Y|Z}; p_Z) = \tsalliscd(\pxz, \pyz) 
   &=\int  p_Z(z)\frac{1}{\alpha-1}\left(\int p_{X|Z}^{\alpha}(u,z)
      p_{Y|Z}^{1-\alpha}(u,z) \ud u -1\right) \ud z  \\
    &= \frac{1}{1- \alpha} + \frac{1}{\alpha-1}
    \int p^{\alpha}_{XZ}(u,z) p^{\beta}_{YZ}(u,z) \ud u \ud z
  \end{align*}
  }
where we have taken $\beta = 1-\alpha$. We have samples
$V_i = (X_i, Z_{1i}) \sim \pxz, i = 1, \dots, n$
and $W_j = (Y_j,Z_{1j})\sim \pyz, j=1, \dots, m$
We will assume $\pxz, \pyz \in \Sigma(s,L,B', B)$.
For brevity, we will write $p = (\pxz, \pyz)$ and $\phat = (\pxzhat, \pyzhat)$.

\subsection{The Estimators}
\label{sec:tsallisEstimator}

We first compute  the influence functions of $\tsalliscd$ and the use
it to derive the \ds/\loos estimators. \\[\thmparaspacing]

\begin{proposition}[Influence Functions of $\tsalliscd$]
\label{thm:tsalliscdInfFun}
The influence functions of $\tsalliscd$ w.r.t $\pxz$, $\pyz$ are
\begingroup
\allowdisplaybreaks
\ifthenelse{\boolean{istwocolumn}}
  {
  \begin{align*}
  &\psixz(X, Z_1; \pxz, \pyz) = \numberthis \label{eqn:tsalliscdInfFun} \\
  &\frac{\alpha}{\alpha-1}  \left(
    \pxz^{\alpha-1}(X,Z_1)\pyz^{\beta}(X,Z_1) - \int \pxz^\alpha\pyz^\beta  \right) \\
  &\psiyz(Y, Z_2; \pxz, \pyz) = \\
  & - \left(
  \pxz^{\alpha}(Y,Z_2)\pyz^{\beta-1}(Y,Z_2) - \int \pxz^\alpha\pyz^\beta  \right)
  \end{align*}
  }
  {
  \begin{align*}
  \psixz(X, Z_1; \pxz, \pyz) &= \numberthis \label{eqn:tsalliscdInfFun} 
  \frac{\alpha}{\alpha-1}  \left(
    \pxz^{\alpha-1}(X,Z_1)\pyz^{\beta}(X,Z_1) - \int \pxz^\alpha\pyz^\beta  \right) \\
  \psiyz(Y, Z_2; \pxz, \pyz) &= 
   - \left(
  \pxz^{\alpha}(Y,Z_2)\pyz^{\beta-1}(Y,Z_2) - \int \pxz^\alpha\pyz^\beta  \right)
  \end{align*}
  }
\endgroup
\begin{proof}
Recall that we can derive the influence functions via
$\psixz(X,Z_1; p) = \tsalliscd'_{XZ}(\delta_{X,Z_1}-\pxz; p)$,
$\psiyz(Y,Z_2; p) = \tsalliscd'_{YZ}(\delta_{X,Z_2}-\pyz; p)$
where $\tsalliscd'_{XZ}, \tsalliscd'_{YZ}$ are the \gateaux derivatives of
$\tsalliscd$
w.r.t $\pxz, \pyz$ respectively. Hence,
\begingroup
\allowdisplaybreaks
\ifthenelse{\boolean{istwocolumn}}
  {
  \begin{align*}
  &\psixz(X,Z_1) = \\
  & \frac{1}{\alpha-1} \frac{\partial}{\partial t}
  \int ((1-t)\pxz + t\delta_{X Z_1})^\alpha \pyz^\beta \Big|_{t=0} \\
  & \frac{\alpha}{\alpha-1}\int\pxz^{\alpha-1} \pyz^\beta (\delta_{X Z_1} - \pxz) 
  \end{align*}
  }
  {
  \begin{align*}
  \psixz(X,Z_1) 
  &= \frac{1}{\alpha-1} \frac{\partial}{\partial t}
  \int ((1-t)\pxz + t\delta_{X Z_1})^\alpha \pyz^\beta \Big|_{t=0} \\
  &= \frac{\alpha}{\alpha-1}\int\pxz^{\alpha-1} \pyz^\beta (\delta_{X Z_1} - \pxz) 
  \end{align*}
  }
\endgroup
from which the result follows. Deriving $\psiyz$ is similar.
Alternatively, we can directly  show that $\psixz, \psiyz$ in
Equation~\eqref{eqn:tsalliscdInfFun}
satisfy Definition~\ref{def:infFun}.
\end{proof}
\end{proposition}
\vspace{\thmparaspacing}

\textbf{\dss estimator}:
Use
$\Vone, \Wone$ to construct density estimates $\pxzhatone, \pyzhatone$ for $\pxz,
\pyz$. Then, use $\Vtwo, \Wtwo$ to add the sample means of the influence
functions given in
Theorem~\ref{thm:tsalliscdInfFun}. This results in our preliminary estimator, 
\begingroup
\allowdisplaybreaks
\ifthenelse{\boolean{istwocolumn}}
  {
  \begin{align*}
  \tsalliscdestimss{1} &= \frac{1}{1 - \alpha} + \frac{\alpha}{\alpha-1}
    \frac{2}{n}\nsumsechalf
    \left( \frac{\pxzhatone(X_i, Z_{1i})}{\pyzhatone(X_i, Z_{1i})} \right)^{\alpha-1} \\
  &\hspace{0.2in}  - \frac{2}{m}\msumsechalf
    \left( \frac{\pxzhatone(Y_j, Z_{2j})}{\pyzhatone(Y_j, Z_{2j})} \right)^{\alpha}
    \numberthis.
  \label{eqn:tsalliscdestim}
  \end{align*}
  }
  {
  \begin{align*}
  \tsalliscdestimss{1} &= \frac{1}{1 - \alpha} + \frac{\alpha}{\alpha-1}
    \frac{2}{n}\nsumsechalf
    \left(\frac{\pxzhatone(X_i, Z_{1i})}{\pyzhatone(X_i, Z_{1i})} \right)^{\alpha-1} 
    - \frac{2}{m}\msumsechalf
    \left(\frac{\pxzhatone(Y_j, Z_{2j})}{\pyzhatone(Y_j, Z_{2j})} \right)^{\alpha}
    \numberthis
  \label{eqn:tsalliscdestim}
  \end{align*}
  }
\endgroup
The final estimate is 
$\tsalliscdestimds = (\tsalliscdestimss{1} + \tsalliscdestimss{2})/2$
where $\tsalliscdestimss{2}$ is obtained by swapping the two samples.

\textbf{\loos Estimator:}
Denote the density estimates of $\pxz, \pyz$ without the $i$\superscript{th} sample by
$\pxzhatmi$ and $\pyzhatmi$. Then the \loos estimator is,
\begin{equation}
\tsalliscdestim = \frac{1}{1-\alpha} + 
  \frac{\alpha}{\alpha-1} \frac{1}{n} \nsumwhole 
    \left(\frac{\pxzhatmi(X_i, Z_{1i})}{\pyzhat(X_i, Z_{1i})} \right)^{\alpha-1} 
   - \left(\frac{\pxzhat(Y_i, Z_{2i})}{\pyzhatmi(Y_i, Z_{2i})} \right)^{\alpha}
\label{eqn:tsalliscdlooestim}
\end{equation}

\subsection{Analysis and Asymptotic Confidence Intervals}
\label{sec:tsallisAnalysis}

We begin with a functional Taylor expansion of $\tsalliscd(f,g)$ around
$(f_0, g_0)$. Since $\alpha, \beta \neq 0, 1$,
we can bound the second order terms by  $O\left( \|f-f_0\|^2 + \|g-g_0\|^2
\right)$. 
\ifthenelse{\boolean{istwocolumn}}
  {
  \begin{align*}
  \tsalliscd(f,g) &=
    \tsalliscd(f_0, g_0) +
    \frac{\alpha}{\alpha - 1} \int f_0^{\alpha-1}g_0^\beta + \numberthis 
  \label{eqn:tsalliscdvme} \\
    & - \int f_0^{\alpha}g_0^{\beta-1} +
    O\left( \|f-f_0\|^2 + \|g-g_0\|^2 \right)
  \end{align*}
  }
  {
  \begin{align*}
  \tsalliscd(f,g) &=
    \tsalliscd(f_0, g_0) +
    \frac{\alpha}{\alpha - 1} \int f_0^{\alpha-1}g_0^\beta  \numberthis 
  \label{eqn:tsalliscdvme} 
     - \int f_0^{\alpha}g_0^{\beta-1} +
    O\left( \|f-f_0\|^2 + \|g-g_0\|^2 \right)
  \end{align*}
  }
Precisely, the second order
remainder is,
\ifthenelse{\boolean{istwocolumn}}
  {
  \begin{align*}
    &\frac{\alpha^2}{\alpha - 1} \int f_*^{\alpha-2}g_*^\beta (f - f_0)^2
    -\beta \int f_*^{\alpha}g_*^{\beta-2} (g - g_0)^2 \\
    &+\frac{\alpha\beta}{\alpha -1} \int f_*^{\alpha-1}g_*^{\beta} (f-f_0)(g-g_0) 
  \end{align*}
  }
  {
  \begin{align*}
    &\frac{\alpha^2}{\alpha - 1} \int f_*^{\alpha-2}g_*^\beta (f - f_0)^2
    -\beta \int f_*^{\alpha}g_*^{\beta-2} (g - g_0)^2 
    +\frac{\alpha\beta}{\alpha -1} \int f_*^{\alpha-1}g_*^{\beta} (f-f_0)(g-g_0) 
  \end{align*}
  }
where $(f_*,g_*)$ is in the line segment between $(f,g)$ and $(f_0,g_0)$.
If $f,g,f_0,g_0$ are bounded above and below so are $f_*, g_*$ and
 $ f_*^a g_*^b $ where $a,b$ are coefficients depending on $\alpha$.
The first two terms are respectively 
$O\left( \|f-f_0\|^2\right)$, $O\left(\|g-g_0\|^2 \right)$.
The cross term can be bounded via,
$\abr{\int (f-f_0)(g-g_0) } \leq \int \max\{ |f-f_0|^2, |g-g_0|^2 \} \in 
O( \|f-f_0\|^2 + \|g-g_0\|^2)$. 

As mentioned earlier, the boundedness of the densities give us the required rates
given in Theorems~\ref{thm:convTwoLoo} for both estimators.

For the \dss estimator, 
to show asymptotic normality, we need to verify the conditions in
Theorem~\ref{thm:asympNormalTwoDistro}.
We state it formally below, but prove it at the end of this section. \\

\begin{corollary}
\label{thm:tsalliscdAsympNormal}
Let $\pxy,\pxz \in \Sigma(s, L, B, B')$. Then $\tsalliscdestimds$
 is asymptotically normal when $\pxz \neq \pyz$ and  $s>d/2$.
\label{thm:condTsallisAsympNormal}
\end{corollary}

Finally, to construct a confidence interval we need a consistent
estimate of the asymptotic variance : 
$\frac{1}{\zeta} \VV_{XZ}\left[ \psixz(V; p) \right]
+ \frac{1}{1-\zeta} \VV_{YZ} \left[ \psiyz(W; p) \right]$ where,
\begingroup
\allowdisplaybreaks
\ifthenelse{\boolean{istwocolumn}}
  {
  \begin{align*}
  & \VV_{XZ}\left[ \psixz(X, Z_1; \pxz, \pyz) \right] = \\ 
  & \left( \frac{\alpha}{\alpha-1} \right)^2 \left( \int \pxz^{2\alpha-1}
  \pyz^{2\beta} - \left( \int \pxz^\alpha \pyz^\beta \right)^2 \right) \\
  & \VV_{YZ}\left[ \psiyz(Y, Z_2; \pxz, \pyz) \right] =  \\
  & \left( \int \pxz^{2\alpha}
  \pyz^{2\beta-1} - \left( \int \pxz^\alpha \pyz^\beta \right)^2 \right)
  \end{align*}
  }
  {
  \begin{align*}
   \VV_{XZ}\left[ \psixz(X, Z_1; \pxz, \pyz) \right] &=
  \left( \frac{\alpha}{\alpha-1} \right)^2 \left( \int \pxz^{2\alpha-1}
  \pyz^{2\beta} - \left( \int \pxz^\alpha \pyz^\beta \right)^2 \right) \\
  \VV_{YZ}\left[ \psiyz(Y, Z_2; \pxz, \pyz) \right] &=  
  \left( \int \pxz^{2\alpha}
  \pyz^{2\beta-1} - \left( \int \pxz^\alpha \pyz^\beta \right)^2 \right)
  \end{align*}
  }
\endgroup
From our analysis above, 
we know that any 
functional of the form $S(a,b) = \int \pxz^a \pyz^b$, $a+b=1, a,b\neq 0,1$
can be estimated via a \loos estimate
\begin{align*}
\widehat{S}(a,b) =
\frac{1}{n}\nsumwhole a \frac{\widehat{p}^b_{YZ,-i}(V_i)}{\widehat{p}^b_{XZ,-i}(V_i)} 
   +  b \frac{\widehat{p}^a_{XZ,-i}(W_i)}{\widehat{p}^a_{YZ,-i}(W_i)} 
\end{align*}
where $\pxzhatmi,\pyzhatmi$ are the density estimates from 
$V_{-i},W_{-i}$ respectively.
$n/N$ is a consistent estimator for $\zeta$.
 This gives the following estimator for the asymptotic variance,
\ifthenelse{\boolean{istwocolumn}}
  {
  \begin{align*}
  &\frac{N}{n}\frac{\alpha^2}{(\alpha-1)^2} \widehat{S}(2\alpha-1, 2\beta)  +
  \frac{N}{m} \widehat{S}(2\alpha, 2\beta-1) \\
  &- 
   \frac{N (m\alpha^2 + n(\alpha-1)^2)}{nm(\alpha-1)^2}\widehat{S}^2(\alpha, \beta).
  \end{align*}
  }
  {
  \begin{align*}
  \frac{N}{n}\frac{\alpha^2}{(\alpha-1)^2} \widehat{S}(2\alpha-1, 2\beta)  +
  \frac{N}{m} \widehat{S}(2\alpha, 2\beta-1) - 
   \frac{N (m\alpha^2 + n(\alpha-1)^2)}{nm(\alpha-1)^2}\widehat{S}^2(\alpha, \beta).
  \end{align*}
  }
The consistency of this estimator follows from the consistency of
$\widehat{S}(a,b)$ for $S(a,b)$, Slutzky's theorem and the
continuous mapping theorem.

\begin{proof}[Proof of Corollary~\ref{thm:condTsallisAsympNormal}]
We now prove that the \dss estimator satisfies the necessary conditions for
asymptotic normality. We begin by showing that $\tsalliscd$'s influence functions
satisfy the regularity condition~\ref{asm:infFunRegularity}.
We will show this for $\psiyz$. The proof for $\psixz$ is similar.
Consider two pairs of densities $(f,g)$ $(f',g')$ on the $(XZ,YZ)$ spaces.
\begingroup
\allowdisplaybreaks
\begin{align*}
& \int \left( \psixz(u; f, g) - \psixz(u; f', g') \right)^2 f\\
&\hspace{0.2in}= \frac{\alpha^2}{(1-\alpha)^2}
\int \left( f^{\alpha-1}g^\beta - \int f^\alpha g^\beta
  -\left[ f'^{\alpha-1}g'^\beta - \int f'^\alpha g'^\beta \right]
  \right)^2 f \\
&\hspace{0.2in}\leq 2\frac{\alpha^2}{(1-\alpha)^2} \left[
\int \left(f^{\alpha-1}g^\beta - f'^{\alpha-1}g'^\beta \right)^2f + 
\left( \int f^\alpha g^\beta - \int f'^\alpha g'^\beta \right)^2
\right] \\
&\hspace{0.2in}\leq 2\frac{\alpha^2}{(1-\alpha)^2} \left[
\int \left(f^{\alpha-1}g^\beta - f'^{\alpha-1}g'^\beta \right)^2f + 
\int \left( f^\alpha g^\beta - f'^\alpha g'^\beta \right)^2
\right] \\
&\hspace{0.2in} \leq 4\frac{\alpha^2}{(1-\alpha)^2} \bigg[
\|g^{\beta}\|_\infty^2 \int (f^{\alpha-1} - f'^{\alpha-1})^2 +
\|f'^{\alpha-1}\|_\infty^2 \int (g^\beta - g'^\beta)^2 + \\
&\hspace{0.6in}
\|g^{\beta}\|_\infty^2 \int (f^{\alpha} - f'^{\alpha})^2 +
\|f'^{\alpha}\|_\infty^2 \int (g^\beta - g'^\beta)^2 
\bigg] \\
&\hspace{0.2in}\in
O\left( \|f -f'\|^2 \right) +
O\left( \|g-g'\|^2 \right) 
\end{align*}
\endgroup
where, in the second and fourth steps we have used Jensen's inequality.
The last step follows from the boundedness of all our densities and estimates
and by lemma~\ref{lem:densitypowers}.

The bounded variance condition of the influence functions also follows from the
boundedness of the densities.
\begingroup
\allowdisplaybreaks
\ifthenelse{\boolean{istwocolumn}}
  {
  \begin{align*}
  &\VV_{\pxz} \psixz(V;\pxz, \pyz) \\
  &\leq \frac{\alpha^2}{ (\alpha-1)^2 }
    \EE_{\pxz} \left[\pxz^{2\alpha-2}(X,Z_1) \pyz^{2\beta}(X,Z_1) \right] \\
  &= \frac{\alpha^2}{ (\alpha-1)^2 }\int \pxz^{2\alpha-1} \pyz^{2\beta} < \infty
  \end{align*}
  }
  {
  \begin{align*}
  \VV_{\pxz} \psixz(V;\pxz, \pyz) 
  &\leq \frac{\alpha^2}{ (\alpha-1)^2 }
    \EE_{\pxz} \left[\pxz^{2\alpha-2}(X,Z_1) \pyz^{2\beta}(X,Z_1) \right] \\
  &= \frac{\alpha^2}{ (\alpha-1)^2 }\int \pxz^{2\alpha-1} \pyz^{2\beta} < \infty
  \end{align*}
  }
\endgroup
We can bound $\VV_{\pyz}\psiyz$ similarly.
For the fourth condition, note that when $\pxz = \pyz$,
\ifthenelse{\boolean{istwocolumn}}
  {
  \begin{align*}
  &\psi_{XZ}(X,Z_1; \pxz, \pxz) \\
  &= \frac{\alpha}{\alpha -1} \left(
  \pxz^{\alpha+\beta-1}(X,Z_1) - \int \pxz\right) 
  = 0,
  \end{align*}
  }
  {
  \begin{align*}
  \psi_{XZ}(X,Z_1; \pxz, \pxz) 
  &= \frac{\alpha}{\alpha -1} \left(
  \pxz^{\alpha+\beta-1}(X,Z_1) - \int \pxz\right) 
  = 0,
  \end{align*}
  }
and similarly $\psi_{YZ} = \zero$.
Otherwise, $\psixz$ depends explicitly on $X,Z$ and is nonzero.
Therefore we have asymptotic normality away from $\pxz = \pyz$.
\end{proof}

\section{Addendum to Experiments}
\label{sec:appExperiments}

\subsection{Details on Simulations}

In our simulations, for the first figure comparing the Shannon Entropy
 in Fig~\ref{fig:toyOne} we generated data
from the following one dimensional density,
\[
f_1(t) = 0.5 + 0.5 t^9
\]
For this, with probability $1/2$ we sample from the uniform distribution $U(0,1)$
on $(0,1)$ and otherwise sample $10$ points from $U(0,1)$ and pick the maximum.
For the third figure in Fig~\ref{fig:toyOne} comparing the KL divergence, we
generate data from the one dimensional density 
\[
f_2(t) = 0.5 + \frac{0.5t^{19}(1-t)^{19}}{B(20,20)}
\]
where $B(\cdot,\cdot)$ is the Beta function. For this, with probability $1/2$ we
sample from $U(0,1)$ and otherwise sample from a $\textrm{Beta}(20,20)$ distribution.
The second and fourth figures of Fig~\ref{fig:toyOne} we sampled from a $2$ dimensional
density where the first dimension was $f_1$ and the second was $U(0,1)$.
The fifth and sixth were from a $2$ dimensional
density where the first dimension was $f_2$ and the second was $U(0,1)$.
In all figures of Fig.~\ref{fig:toyTwo}, the first distribution was a $4$-dimensional density
where all dimensions are $f_2$. The latter was $U(0,1)^4$.

\textbf{Methods compared to: }
In addition to the plug-in, \dss and \loos estimators we perform comparisons with
several other estimators.
For the Shannon Entropy we compare our method to the \knns estimator of
\citet{goria2005new}, the method of \citet{stowell2009fast} which uses $K-D$
partitioning, the method of \citet{noughabi2013entropy} based on Vasicek's spacing
method and that of \citet{learned2003ica} based on Voronoi tessellation.
For the KL divergence we compare against the \knns method of \citet{perez2008kullback}
and that of \citet{ramirez2009entropy} based on the power spectral density
representation using Szego's theorem. For \renyi, \tsallis and Hellinger divergences
we compared against the \knns method of \citet{poczos12divergence}.

\subsection{Image Clustering Task}

Here we demonstrate a simple image clustering task using a nonparametric divergence
estimator. For this we use images from the ETH-80 dataset.
The objective here is not to champion our approach for image clustering against
all methods for image clustering out there. 
Rather, we just wish to demonstrate that our estimators can
be easily and intuitively applied to many Machine Learning problems.

We use the three categories Apples, Cows and Cups and randomly select $50$ images
from each category. Some sample images are shown in Fig~\ref{fig:clusImages}.
We convert the images to grey scale and extract the SIFT features from each
image. The SIFT features are $128$-dimensional but we project it to $4$ dimensions
via PCA. This is necessary because nonparametric methods work best in low dimensions.
Now we can treat each image as a collection of features, and hence a sample from a $4$
dimensional distribution. 
We estimate the Hellinger divergence between these ``distributions".
Then we construct an affinity matrix $A$ where the similarity metric between the
$i$\superscript{th} and $j$\superscript{th} image is given by $A_{ij} =
\exp(-\widehat{H}^2(X_i, X_j))$. Here $X_i$ and $X_j$ denotes the projected SIFT
samples from images $i$ and $j$ and $\widehat{H}(X_i, X_j)$ is the estimated
Hellinger divergence between the distributions.
Finally, we run a spectral clustering algorithm on the matrix $A$.

Figure~\ref{fig:affinity} depicts the affinity matrix $A$ when the images were
ordered according to their class label. The affinity matrix exhibits 
block-diagonal structure which indicates that our Hellinger divergence estimator
can in fact identify patterns in the images. 
Our approach achieved a clustering accuracy of $92.47\%$. When we used the $k$-NN
based estimator of~\cite{poczos12divergence} we achieved an accuracy of
$90.04\%$.
When we instead applied Spectral clustering naively, 
with $A_{ij} = \exp(-L_2(P_i,P_j)^2)$ where $L_2(P_i,P_j)$ is
the squared $L_2$ distance between the pixel intensities we achieved an accuracy of
$70.18\%$. We also tried $A_{ij} = \exp(-\alpha \widehat{H}^2(X_i, X_j))$ as the
affinity for
different choices of $\alpha$ and found that our estimator still performed best.
We also experimented with the
\renyi and \tsallis divergences and obtained similar results.

On the same note, one can imagine that these divergence estimators can also be used
for a classification task. For instance we can treat $\exp(-\widehat{H}^2(X_i, X_j))$
as a similarity
metric between the images and use it in a classifier such as an SVM.

\insertFigClustering

\end{document}